\documentclass[runningheads]{llncs}
\usepackage{amsmath,amsfonts}
\usepackage{algorithm}
\usepackage{algorithmicx} 
\usepackage
{algpseudocode}
\usepackage{graphicx}
\usepackage{textcomp}
\usepackage{xcolor}
\usepackage{gensymb}
\usepackage{siunitx}
\usepackage{xspace}
\usepackage[english]{babel}
\usepackage{todonotes}
\usepackage{booktabs}
\usepackage{array}
\usepackage{hyperref}
\usepackage{mathtools}
\usepackage{array}
\usepackage{multirow}
\usepackage{tabularx} 
\usepackage{soul}
\usepackage{comment}
\usepackage{hhline}
\usepackage{url}
\usepackage{cite}
\usepackage[strings]{underscore} 
\usepackage{bm}
\usepackage{subfigure}
\usepackage{colortbl}

\newcolumntype{C}[1]{>{\centering\arraybackslash}p{#1}}

\newcommand{\chd}{\ensuremath{c_{hd}}\xspace} 

\newcommand{\cld}{\ensuremath{c_{ld}}\xspace}  

\newcommand{\mld}{\ensuremath{M_{ld}}\xspace} 

\newcommand{\mhd}{\ensuremath{M_{hd}}\xspace} 

\newcommand{\tld}{\ensuremath{\tau_{ld}}\xspace} 

\newcommand{\thd}{\ensuremath{\tau_{hd}}\xspace} 

\newcommand{\rs}{\mathsf{rs}} 

\newcommand{\rt}{\mathsf{rt}} 

\newcommand{\irs}{\mathsf{irs}} 

\newcommand{\irt}{\mathsf{irt}} 

\newcommand{\bgamma}{\bar{\gamma}} 

\newcommand{\prob}{\operatorname{Pr}} 

\newcommand{\DS}{\ensuremath{D_0}\xspace} 
\newcommand{\Di}{\ensuremath{D_i}\xspace} 

\begin{document}

\title{Bridging Dimensions: Confident Reachability for High-Dimensional Controllers}
\author{Anonymous Authors}
\institute{Anonymous Institution}
\titlerunning{Bridging Dimensions: Confident Reachability for HD Controllers}

\vspace{-4mm}

\author{Yuang Geng \and Jake 
Baldauf \and
Souradeep Dutta \and Chao Huang \and
Ivan Ruchkin}
\authorrunning{Geng et al.}
%

\institute{University of Florida, FL, USA, \email{yuang.geng@ufl.edu}  \\
\and
University of Florida, FL, USA, \email{jakebaldauf@ufl.edu}  \\
\and
University of Pennsylvania, PA, USA, \email{duttaso@seas.upenn.edu} \\
\and
University of Southampton, UK, \email{chao.huang@soton.ac.uk} \\
\and
University of Florida, FL, USA,
\email{iruchkin@ece.ufl.edu} 
}

%
%

%
\maketitle              
\vspace{-6mm}
\begin{abstract}
Autonomous systems are increasingly implemented using end-to-end learning-based controllers. Such controllers make decisions that are executed on the real system, with images as one of the primary sensing modalities. Deep neural networks form a fundamental building block of such controllers. Unfortunately, the existing neural-network verification tools do not scale to inputs with thousands of dimensions --- especially when the individual inputs (such as pixels) are devoid of clear physical meaning. This paper takes a step towards connecting exhaustive closed-loop verification with high-dimensional controllers. Our key insight is that the behavior of a high-dimensional controller can be approximated with several low-dimensional controllers.
To balance the approximation accuracy and verifiability of our low-dimensional controllers, we leverage the latest verification-aware knowledge distillation. Then, we inflate low-dimensional reachability results with statistical approximation errors, yielding a high-confidence reachability guarantee for the high-dimensional controller. We investigate two inflation techniques --- based on trajectories and control actions --- both of which show convincing performance in three OpenAI gym benchmarks. 

\vspace{-1mm}

\looseness=-1
\end{abstract}
\vspace{-3mm}

\vspace{-7mm}
\section{Introduction}
\vspace{-1mm}

End-to-end deep neural network controllers have been extensively used in executing complex and safety-critical autonomous systems in recent years~\cite{codevilla_end--end_2018,matsumoto_end--end_2020,topcu_assured_2020,teeti_vision-based_2022}.  In particular,\textit{ high-dimensional controllers} (HDCs) based on images and other high-dimensional inputs have been applied in areas such as autonomous car navigation~\cite{third_eye,deeproad} and aircraft landing guidance~\cite{NNVerifier}. For example, recent work has shown the high performance of controlling aircraft to land on the runway with a vision-based controller~\cite{vision_landing}. For such critical applications, it is important to develop techniques with strong safety guarantees for HDC-controlled systems. 

\looseness=-1
However, due to the high-dimensional nature of the input space, modern verification cannot be applied directly to systems controlled by HDCs~\cite{closed_analysis_vision,CORA}. Current closed-loop verification tools, such as  NNV~\cite{nnv}, Verisig~\cite{verisig}, Sherlock \cite{sherlock}, and ReachNN*~\cite{reachnn}, are capable of combining a dynamical system and a \textit{low-dimensional controller} (LDC) to verify a safety property starting from an initial region of the low-dimensional input space, such as position-velocity states 
of a car. DeepReach~\cite{deepreach} has pushed the boundary of applying Hamilton-Jacobi (HJ) reachability to systems with tens of state dimensions. 
However, such verification tools fail to scale for an input with thousands of dimensions (e.g., an image). One issue is that the dynamics of these dimensions are impractical to describe. Furthermore, the structure of an HDC is usually more complicated than that of an LDC, with convolution and pooling layers. For example, an image-based HDC may have hundreds of layers with thousands of neurons, whereas an LDC usually contains several layers with dozens of neurons, making HDC verification difficult.

\looseness=-1
To deal with these challenges, researchers have built abstractions of high-dimensional perception into the verification process. One work~\cite{gan} verified a generative adversarial network (GAN) that creates images from states. 
 Such methods cannot guarantee the GAN's accuracy or relation to reality, which becomes a major falsifiable assumption of their verification outcomes. Another work~\cite{NNVerifier} built a precise mathematical model capturing the exact relationship between states and image pixels to verify the image-based controller, which is effortful and needs to be redone for each system. Inspired by previous work on decreasing the dimensions, we skillfully create verifiable low-dimensional controllers from high-dimensional ones. 

\looseness=-1
This paper proposes an \textbf{end-to-end methodology} to verify systems with HDCs by employing the steps displayed in Fig.~\ref{fig:high-level}. Instead of verifying an HDC's safety directly over a complicated input space, our key idea is to approximate it with several LDCs so that we can reduce the HDC reachability problem to several LDC reachability problems. A crucial step is to upper-bound the difference between LDC and HDC, which we do statistically. Finally, we extend the reachable sets with the statistical bounds to obtain a safety guarantee for the HDC. 

\begin{figure*}[t]
\centerline{\includegraphics[width=\columnwidth]{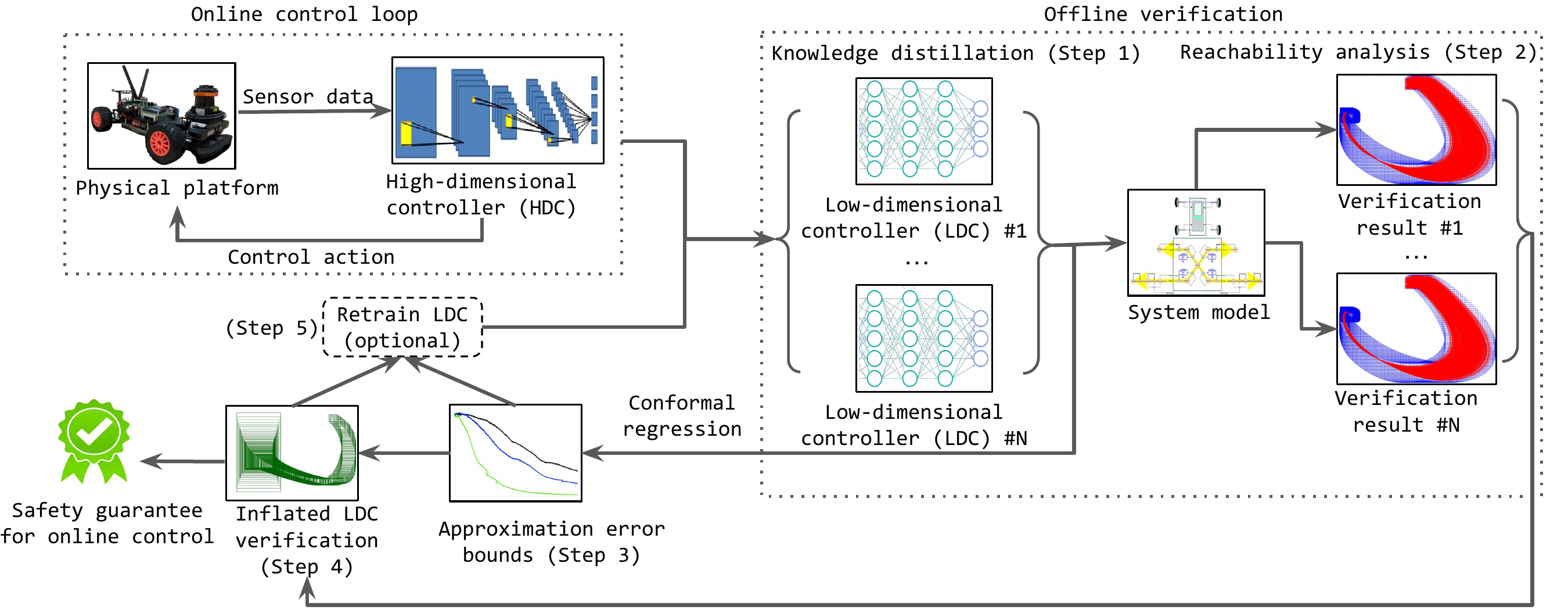}}
\vspace{-4mm}
\caption{Our verification approach for systems with high-dimensional controllers. 
}
\vspace{-7mm}
\label{fig:high-level}
\end{figure*}

\looseness=-1
Since the input space and structure of the HDC are too complex to verify, we leverage \textit{knowledge distillation}~\cite{distill}---a model compression method---to train simplified ``student models'' (LDC) based on the information from the sophisticated ``teacher model'' (HDC). This training produces an LDC that is lightweight and amenable to closed-loop verification because it operates on dynamical states, not images. Moreover, due to the importance of the Lipschitz to minimizing the overapproximation error~\cite{lower_lips,reachnn}, our methodology adopts \textit{two-objective gradient descent}~\cite{Fan_kd}, decreasing both the approximation error and Lipschitz constant.

After training the LDCs, we calculate the statistical upper bound of the discrepancy between the two controllers, since obtaining the true discrepancy is impractical. To this end, we rely on \textit{conformal prediction}~\cite{CP_lemma1,CP,shafer_tutorial_2008}, one of the cutting-edge statistical methods to provide a lower bound of the confidence interval for prediction residuals without distributional assumptions or explicit dependency on the sample count. We propose \textit{two conformal techniques} to quantify the difference between HDC- and LDC-controlled systems, by bounding: (i) the distance between their trajectories, and (ii) the difference between the actions produced by the HDC and LDC. We inflate reachable sets of the LDC system based on both bounds to obtain safety guarantees on the HDC system.


We evaluate our methodology on three popular verification case studies in OpenAI Gym~\cite{brockman_openai_2016}: inverted pendulum, mountain car, and cartpole. 
While our results are promising, verifying HDCs remains a challenging task with many improvement opportunities. The contributions of our work are three-fold:
\begin{enumerate} 
    \item Two verification approaches for high-dimensional controllers that combine reachability analysis and statistical inference to provide a safety guarantee for systems controlled by neural networks with thousands of inputs. 
    \item A novel neural-network approximation technique for training multiple LDCs that collectively mimic an HDC and reduce overapproximation error.
    \item An implementation and evaluation of our verification approaches on three case studies: inverted pendulum, mountain car, and cartpole.
\end{enumerate}

Sec.~\ref{sec:background} introduces the background and defines our verification problem. Sec.~\ref{sec:approach} describes the details of our verification approach, which is evaluated in Sec.~\ref{sec:evaluation}. Finally, we review the related work in Sec.~\ref{sec:relwork} and conclude the paper in Sec.~\ref{sec:conclusion}.

\vspace{-3mm}
\section{Background and Problem Setting} \label{sec:background}
\vspace{-2mm}

\looseness=-1
\textbf{High- and low-dimensional systems.}
The original \textit{high-dimensional closed-loop system} is a tuple $\mhd= (S, Z, U, s_0, f, \chd, g)$. Here, the $S$ is the state space, $Z$ is the high-dimensional sensor space of so-called ``images'' (e.g., camera images or LIDAR scans), and the $U$ is the control action space, $s_0$ is the initial state, $f:S \times U \rightarrow S$ is the dynamics, 
and
$\chd: Z \times S \rightarrow U$ is the HDC. Note that \chd only uses a \textit{subset} of state dimensions as input (e.g., a convolutional neural network with image and velocity inputs, but not position), getting the rest of the information from the image. 

For mathematical convenience, we also define an (unknown) deterministic state-to-image generator as $g: S \rightarrow Z$ and the role and assumptions of generator $g$ are stated below. As a verifiable approximation of \mhd, our \textit{low-dimensional closed-loop system} is defined as  $\mld = (S, U, s_0, f, \cld)$. Both \mhd and \mld have the same state space and action space. The only difference is that the $M_{ld}$ has a low-dimensional controller $\cld: S \rightarrow U$, which operates on the exact states.


\noindent
\textbf{System execution.}
The execution of $M_{hd}$  starts from the initial state $s_0$. Next, an image $z$ can be generated by image generator $g$ from that state. Then it is fed into \chd to obtain a corresponding control action $u = \chd(z)$, which is used to update the state via dynamics $f$. For $M_{ld}$, the execution proceeds similarly, except that the current state $s$ directly results in a control action $u = \cld(s)$. Thus, we denote the \textit{state at time} $t$ starting from $s_0$ executed by $\mhd$ or $\mld$  as $\varphi_{hd}(s_0, t)$ and $\varphi_{ld}(s_0, t)$ respectively. The  \textit{trajectory} of \mhd is defined as a state sequence: $\tau_{hd}(s_0, T) = [s_0,  \varphi_{hd}(s_0, 1), \dots,  \varphi_{hd}(s_0, T)],$ and similarly for $\tau_{hd}$. 

Based on previous background, we define reachable sets and tubes:
\begin{definition}[Reachable set]
\label{def:reachable_set} Given an initial set $S_0$ and an integer time $t$, a \emph{reachable set} $\rs_M(S_0, t)$ for (either) system $M$ contains all the states that can be reached from $S_0$ in $t$ steps: $\rs_M(S_0, t) = \{\varphi_M(s_0, t) \mid \forall s_0 \in S_0\}$.
\end{definition}

\begin{definition}[Reachable tube] Given an initial set $S_0$ and time horizon $T$, a \emph{reachable tube} $\rt_M (S_0, T)$ for (either) system $M$ is a sequence of all the reachable sets from $S_0$ until time $T$: $\rt_M (S_0, T) = [S_0, \rs_M(S_0,1),..., \rs_M(S_0, T)]$.
\end{definition}

\vspace{-2mm}
\looseness=-1
\noindent
\textbf{Assumptions on image-state mapping $g$.} 
Our key challenge is establishing a mapping between the high-dimensional image space $Z$ and the low-dimensional state space $S$. Our verification methodology is based on the existence of a deterministic image generator $g$ that is part of \mhd. This generator is the true and \textit{unknown} mechanism that creates images from states (e.g., a camera system). We do \textit{not} assume or use an analyzable closed-form description of $g$. We also do not assume or verify any perception model (which obtains states from images).

Our only use of $g$ is in the training process for a limited dataset: we assume that we can ``lab-study'' an instrumented system \mhd (e.g., with extra positioning sensors or human annotations) to label each image $z$ with a corresponding low-dimensional state $s$. In practice, this data collection would correspond to executing an autonomous system inside a controlled test facility. It's crucial to clarify that we are not using $g$ to generate any images from any states. Instead, we solely produce a restricted state-image paired dataset for training LDC.

To check the robustness of the above assumption in realistic scenarios, we will perform a sensitivity analysis by adding zero-mean Gaussian noise to the state-image mapping. 
For simplicity, this noise is added to the true state before generating an image from it, instead of adding pixel noise that may result in unrealistic images. The results of this evaluation will be discussed in Sec.~\ref{sec:evaluation}. 



\looseness=-1
\noindent
\textbf{Verification problem}.
Our problem is to guarantee that the high-dimensional system $M_{hd}$ reaches the goal set $G$ from an initial set $S_0$ within time $T$. 
To this end, we aim to compute reachable sets of the high-dimensional system \mhd and intersect them with the goal set to obtain the verification verdict. 
Set $G$ is specified in low dimensions (i.e., using physical variables); however, the $M_{hd}$ behavior is determined by the images from generator $g$ and the HDC's response to them. 


\looseness=-1
Thus, given initial set $S_0$, goal set $G$, 
system \mhd, and time horizon $T$, \ul{our goal is to verify this assertion}:
\begin{align}\label{eq:verif-problem}
\begin{split}
    \forall s_0 \in S_0 ~\cdot~ &\rs_{\mhd}(S_0, T) \subseteq  G 
\end{split}
\end{align}

\looseness=-1
This problem can be divided into two parts
: (a) approximating \mhd with low-dimensional systems $\mld^1, \dots, \mld^n$ and verifying them; (b)  combining these reachability results based on the approximation error bounds into a reachability verdict to solve the above \mhd problem with statistical confidence.

\vspace{-3mm}
\section{Verification of High-Dimensional Systems}\label{sec:approach}

\looseness=-1

\vspace{-1mm}
Considering the challenges of complex structure and dynamics of high-dimensional systems, and the difficulties of defining safety in high dimensions, our end-to-end approach is structured in \underline{five steps}: (1) train low-dimensional controller(s), (2) perform reachability analysis on them, (3) compute statistical discrepancy bounds between high- and low-dimensional controllers, (4) inflate the reachable tubes from low-dimensional verification with these bounds, and (5) combine the verification results and repeat the process as if needed on different states/LDCs. 

\vspace{-1mm}
\subsection*{Step 1: Training low-dimensional controllers}
Given the aforementioned challenges of directly verifying \mhd, we plan to first verify the behavior in the low dimensions according to \mld. Hence, we train a \cld to imitate the performance of \chd starting from a given state region, which serves as an input to Step 1 (our first iteration uses the full initial state region $S_0$ to train one \cld). 
As a start, we collect the training data for \cld: given the \chd, access to image generator $g$, and the initial state space region $S_0$, we construct a supervised training dataset $\mathcal{D}_{tr} = \big\{\big(\tau_{hd}(s_i,T), (u_1, ..., u_T)_i\big)\big\}_{i=1}^m$ by sampling the initial states $s_i  \sim D_0$ from some given distribution $D_0$ (in practice, $D_0 = \operatorname{Uniform}(S_0)$).

Training a verifiable LDC has \ul{two conflicting objectives}. On the one hand, we want to approximate the given \chd with minimal Mean Squared Error (MSE) on $\mathcal{D}_{tr}$. On the other hand, neural networks with smaller Lipschitz constants (Def.~\ref{def:lp_constant} in the Appendix) are more predictable and verifiable~\cite{szegedy_intriguing_2014,fazlyab_efficient_2019,combettes_lipschitz_2020}.  

We balance the ability of the \cld to mimic the \chd and the verifiability of \cld by using a recent \textit{verification-aware knowledge distillation technique}~\cite{Fan_kd}. Originally, this method was developed to compress low-dimensional neural networks for better verifiability --- and we extend it to approximate an HDC with LDCs using the supervised dataset $\mathcal{D}_{tr}$.  Specifically, we implement knowledge distillation with \textit{two-objective gradient descent}, which aims to optimize the MSE loss function $L_{mse}$ and Lipschitz constant loss function $L_{lip}$. First, it computes the directions of two gradients with respect to the \cld parameters $\theta$:
\begin{align}
d_{L_{mse}} = \frac{\partial L_{mse}}{\partial \theta}, \quad d_{L_{lip}} = \frac{\partial L_{lip}}{\partial \theta}
\end{align}

\looseness=-1
The two-objective descent operates case-by-case to optimize at least one objective as long as possible. If $d_{L_{mse}} \cdot d_{L_{lip}} > 0 $, the objectives can be optimized simultaneously by following the direction of the angular bisector of the two gradients. If $d_{L_{mse}} \cdot d_{L_{lip}} < 0$, then it is impossible to improve both objectives. Then, weights are updated along the vector of $d_{L_{mse}}$ (the higher priority) projected onto the hyperplane perpendicular to $d_{L_{lip}}$. The thresholds for MSE and Lipschitz constants in our system \mld are denoted as $\epsilon$ and $\lambda$ respectively. The stopping condition is met when both loss functions are below their thresholds or the training time exceeds the limit. Later on, Step 1 will be referred to with function \textsc{TrainLDC}, and our way of tuning $\epsilon$ and $\lambda$ will be described later in Step 5.


\subsection*{Step 2: Reachability analysis in low dimensions }


After training LDCs $\{\cld^1,...,\cld^m \}$, we construct overapproximate reachable tubes for each. We perform reachability analysis for systems $\mld^1, \dots, \mld^m$ with the respective controllers and the initial set $S_0$ specified in the original verification problem. This will result in a set of reachable tubes $\rt_{\mld^1}(S_0, T), \dots, \rt_{\mld^m}(S_0, T)$.

To carry out the reachability analysis, we use the \textit{POLAR\footnote{\url{https://github.com/ChaoHuang2018/POLAR_Tool}} toolbox}, version of December 2022~\cite{polar,polar23} --- an implementation that computes univariate Bernstein polynomials to overapproximate activation functions in \cld, and then selectively uses Taylor or Bernstein polynomials for tight overapproximation of \cld. For dynamics reachability, which alternates with neural-network overapproximation, the POLAR toolbox relies on the mature Flow* tool with Taylor model approximations~\cite{flow}. The latest experimental results~\cite{polar23} show that POLAR outperforms other neural-network verification tools in terms of both computational efficiency and tightness. The verification details are formalized in Algorithm~\ref{alg:overall} in Step 5.


\vspace{-2mm}
\subsection*{Step 3a: Defining discrepancy bounds}
 
\looseness=-1
The LDC reachable tubes from Step 2 cannot be used directly to obtain HDC guarantees because of the discrepancy between LDC and HDC behaviors, which inevitably arises when compressing a higher-parameter neural network~\cite{kd_survey}. 
Therefore, we will quantify the difference between LDCs and HDCs using \textit{discrepancy functions}, inspired by the prior work on testing hybrid systems~\cite{fan_bounded_2015,fan_dryvr_2017,qin-conformaltesting}. We introduce and investigate \ul{two types of discrepancy functions} in our setting:


\smallskip
\noindent
 \textbf{1. Trajectory-based discrepancy} $\beta$ considers the difference between the HDC and LDC \textit{trajectories} starting from a \textit{matched} state-image pair $(s,z)$, i.e., $z = g(s)$. It is defined as the least upper bound on the maximum L1 distance between two trajectories, i.e., $\| \thd(s_0,T) - \tld(s_0,T) \|_1$, over time $T$ for all initial states $s_0$ within the initial set $S_0$. Therefore, each initial set $S_0$ gives rise to its trajectory-based discrepancy $\beta(S_0)$.

\smallskip
\noindent
\textbf{2. Action-based discrepancy} $\gamma$ considers the difference between LDC and HDC \textit{actions} on a \textit{matched} state-image pair $(s,z)$, i.e., $z = g(s)$.
Similarly to the above, it is defined as the least upper bound on the difference between control actions over time horizon $T$ starting from any initial state $s_0$ within the initial set $S_0$. Note that the control difference, $\|\chd(g(s_{hd}^t)) - \cld(s_{ld}^t)\|_1$, is considered at each time step, where the $s$ is each state in the two trajectories. 

The formal Definitions~\ref{def:real_trajectory_dis} and~\ref{def:real_action_dis} can be found in the Appendix.



\subsection*{Step 3b: Computing statistical discrepancy bounds}

\looseness=-1
Unfortunately, obtaining the true discrepancies is impractical: it would require solving optimization/feasibility problems in high-dimensional image spaces. Instead, we calculate the statistical upper bounds for these discrepancies via \textit{conformal prediction}, which is a distribution-free statistical technique to provide probabilistically valid uncertainty regions for complex prediction models --- without strong assumptions about these models or their error distributions~\cite{vovk_algorithmic_2005}. 

\looseness=-1
Below we briefly summarize basic conformal prediction. Consider $k+1$ independent and identically distributed random variables $\Delta, \Delta^1,..., \Delta^k$, also known as \textit{non-conformity scores}. Conformal prediction computes an uncertainty region for $\Delta$ via a function $\bar{\Delta}: \mathbb{R}^k \rightarrow \mathbb{R}$ from the other $k$ values. Given a failure probability $\alpha \in (0, 1)$, conformal prediction provides an uncertainty bound on $\bar{\Delta}$ such that $\prob(\Delta \leq \bar{\Delta}$) $\geq 1 - \alpha$.  This is performed with a surprisingly simple quantile argument, where the uncertainty bound $\bar{\Delta}$ is calculated as the $(1-\alpha)$-th
quantile of the empirical distribution over the values of $\Delta^1, \Delta^2,..., \Delta^k,$ and $\infty$. The guarantee is formalized in the lemma below, and for details see a popular tutorial~\cite{shafer_tutorial_2008}. 

\begin{lemma} (Lemma 1 in \cite{CP_lemma1})\label{lemma:org-cp}
    Let $\Delta, \Delta^1, \Delta^2,..., \Delta^k$ be k+1 independent identically distributed real-valued random variables. Without loss of generality, let $\Delta, \Delta^1, \Delta^2,..., \Delta^k$ be stored in non-decreasing order and define $\Delta^{k+1} \coloneqq \infty$. For $\alpha \in (0, 1)$, it holds that $\prob(\Delta\leq \bar{\Delta}$) $\geq 1 - \alpha$ where
    $\bar{\Delta} \coloneqq \Delta^{(r)}$, which is the $r$-ranked variable with $r = \lceil (k+1)(1-\alpha) \rceil$, and $\lceil . \rceil$ is the ceiling function.
\end{lemma}

\looseness=-1
Leveraging conformal prediction, we define the \ul{statistical versions} of our discrepancy functions. For the trajectory-based one, we define the non-conformity as the maximum L1 distance between states at the same time in two matched trajectories $\tld(s_0, T)$ and $\thd(s_0, T)$ starting from a random state $s_0 \sim D_0$ sampled independently and identically distributed (i.i.d.) from a given distribution $D_0$ over the initial region $S_0$, similar to recent works~\cite{cleaveland-traj,qin-conformaltesting}. 
This leads to a trajectory dataset $\mathcal{D}_{tb}$, from which $k$ non-conformity scores are calculated. 


\begin{definition}[Statistical trajectory-based discrepancy]
\looseness=-1
Given 
distribution $D_0$ over $S_0$, confidence $\alpha \in (0, 1)$, and state functions $\varphi_{hd}(s, t)$ and $\varphi_{ld}(s, t)$ for systems $\mhd$ and $\mld$, a \emph{statistical trajectory-based discrepancy} $\bar{\beta}(D_0)$ is an $\alpha$-confident upper bound on the max trajectory distance starting from $s_0\sim D_0$:  
{\small
\begin{align*}&\prob_{s_0 \sim D_0}\Big[   \max_{t=0..T} \| \varphi_{hd}(s_0, t) -\varphi_{ld}(s_0, t) \|_1 \leq \bar{\beta}(D_0)\Big] \geq 1 - \alpha
\end{align*}
 }%
\end{definition}

\looseness=-1
To obtain this bound $\bar{\beta}(D_0)$, we leverage conformal prediction as follows. Dataset $\mathcal{D}_{tb}$ contains i.i.d. samples $s_1, s_2, ..., s_k$ from our chosen distribution $D_0$. In practice, we choose the uniform distribution, namely $s \sim \operatorname{Uniform}(S)$, because we value the safety of each state equally. We compute the corresponding non-conformity scores $\delta^1, \delta^2,..., \delta^k, \delta^{k+1}$ as the maximum L1 distances between the same-time states in the two trajectories over all times $t \in [0..T]$: 
  \begin{align*}
     \delta^i = \max_{t=0..T}\|\varphi_{hd}(s_i, t) -\varphi_{ld}(s_i, t)\|_1  \text{ for } i = 1\dots k;    \text{ and } \delta^{k+1} = \infty 
 \end{align*}

\looseness=-1
We sort the scores in the increasing order and set $\bar{\beta}(S_0)$ to the $r$-th quantile: 
\begin{align}
\bar{\beta}(D_0) \coloneqq \delta^{(r)} \text{ with } r = \lceil (k+1)(1-\alpha) \rceil 
\end{align}

\looseness=-1
We follow a similar procedure for the statistical action-based discrepancy, except that now the non-conformity scores are defined as the maximum differences between actions at the same time in two paired trajectories. 
\vspace{-1mm}
\begin{definition}[Statistical action-based discrepancy]
Given confidence $\alpha \in (0, 1)$, distribution $D_0$ over $S_0$, and systems \mld and \mhd, a \emph{statistical action-based discrepancy} $\bar{\gamma}(D_0)$ is  an $\alpha$-confident upper bound on maximum action discrepancy in two trajectories starting from $s_0\sim D_0$:
\begin{align*}
 \prob_{D(S_0)}\Big[\max_{t=0..T} \|\chd\big(g(\varphi_{hd}(s_0, t))\big) - \cld\big(\varphi_{ld}(s_0, t)\big)\|_1 \leq \bar{\gamma}(D_0)\Big] \geq 1 - \alpha
\end{align*}
\end{definition}

\looseness=-1
To implement this statistical action-based discrepancy function, we sample initial states $s_1, s_2, ..., s_k$ from a given set $S_0$ following the distribution $D_0$ (in practice, uniform) and obtain the corresponding low-dimensional trajectories. Then we generate with $g$ the corresponding images matched to each state in each trajectory --- and these pairs form our action-based dataset $\mathcal{D}_{ab}$. 
The corresponding nonconformity scores $\delta^1, \delta^2,..., \delta^k, \delta^{k+1}$ are maximum action differences:
\begin{align*}
    \delta^i = \max_{t=0..T} \|\chd(g(\varphi_{hd}(s_0, t))) - \cld(\varphi_{ld}(s_0, t))\|_1 \text{ for } i=1\dots k; \delta^{k+1} = \infty.
\end{align*}

Then we sort these non-conformity scores in the non-decreasing order and determine the statistical bound for the action-based discrepancy as:
\begin{align}
\bar{\gamma}(D_0) \coloneqq \delta^{(r)} \text{ with }r = \lceil (k+1)(1-\alpha) \rceil
\end{align}

\vspace{-6mm}
\subsection*{Step 4: Inflating reachability with discrepancies}

\vspace{-2mm}
This step combines low-dimensional reachable tubes (Step 2) with statistical discrepancies (Step 3b) to provide a safety guarantee on the high-dimensional system. 
Thus, we inflate the original LDC reach tubes with either trajectory or action discrepancy to contain the (unknown) true HDC tube with chance $1-\alpha$.

\noindent
\textbf{Trajectory-based inflation.}
The trajectory-based approach inflates the LDC reachable set starting in region $S_0$ with the statistical trajectory-based discrepancy $\bar{\beta}(D_0)$. 
Since the final reachable tube for a given initial set of \cld is represented as a sequence of discrete state polytopes calculated by concretizing the Taylor model with interval arithmetic on the initial set~\cite{polar}, we inflate these polygons by adding $\bar{\beta}(D_0)$ to their boundaries. 

\begin{definition}[Trajectory-inflated reachable set
] \label{def:traj-inflated-set}
Given a distribution $D_0$ over initial set $S_0$ that is controlled by LDC \cld, reachable set $\rs(S_0, t)$, and its trajectory discrepancy $\bar{\beta}( D_0 )$,  a \emph{trajectory-inflated reachable set} is defined as: 
        \begin{align*}
        \irs(S_0, t, \bar{\beta}(D_0)) = \big\{s \in S ~|~ \exists s' \in \rs(S_0, t) \cdot         
        \|s-s'\|_1 \leq \bar{\beta}(D_0)\big\}
        \end{align*}
\end{definition}



\begin{definition} [Trajectory-inflated reachable tube] \label{def:traj-inflated-tube}
Given a distribution $D_0$ over initial set $S_0$ that is controlled by LDC \cld, a reachable tube $\rt(S_0, t) = \big[S_0, \rs(S_0, 1), \dots, \rs(S_0, T)\big]$ over time horizon $T$, and its trajectory discrepancy $\bar{\beta}(D_0)$ over the initial set $S_0$, a \emph{trajectory-inflated reachable tube} $\irt(S_0, \bar{\beta}(D_0)) $ is defined as: 
    \[
    \irt(S_0, \bar{\beta}(D_0))  = \left[\irs(S_0, 0, \bar{\beta}(D_0)), \irs(S_0, 1, \bar{\beta}(D_0)) , \dots ,  \irs(S_0, T, \bar{\beta}(D_0))\right].
    \]
    
\end{definition}
Based on Defs \ref{def:traj-inflated-set} and \ref{def:traj-inflated-tube}, we establish Theorem \ref{thm:tb} that the trajectory-inflated LDC reachable tube contains the HDC reachable tube with at least $1-\alpha$ probability. 

\begin{theorem}[Confident trajectory-based overapproximation]\label{thm:tb} 
Consider distribution $D_0$ over initial set $S_0$, confidence $\alpha$, a high-dimensional system \mhd, approximated with a low-dimensional system controlled by $\cld$ with an $\alpha$-confident statistical trajectory-based discrepancy function $\bar{\beta}(S_0)$. Then the trajectory-inflated low-dimensional tube $\irt_{\mld}(S_0, \bar{\beta}(D_0))$ contains the high-dimensional reachable tube $\rt_{\mhd}(S_0)$ with probability $1-\alpha$:
$$ \prob_{D_0}\Big[\rt_{\mhd}(S_0) \subseteq \irt_{\mld}(S_0, \bar{\beta}(S_0)) \Big] \geq 1-\alpha $$ 
\end{theorem}
\begin{proof}
    All the proofs are found in the Appendix.
\end{proof}

\looseness=-1
Defs.~\ref{def:traj-inflated-set} and \ref{def:traj-inflated-tube} and Thm.~\ref{thm:tb} describe inflation and guarantees with a \textit{single LDC}. 
However, one LDC usually cannot mimic the behavior of the HDC accurately. Therefore, we train several LDCs $\{\cld^1,\cld^2, \dots, \cld^{m} \}$, one for each subregion of initial set $\{S_1, S_2,\dots,S_{m}\}$ with respective distributions $D_0 = \{D_1, D_2,\dots,D_{m}\}$. Subsequently, the trajectory-inflated tube with multiple LDCs can be represented as a union of all the single trajectory-inflated tube
$\irt(S_0, \bar{\beta}(D_0)) := \bigcup_{i = 1}^{m}\irt(S_i, \bar{\beta}(D_i))$. The definitions and the multi-LDC version of Thm.~\ref{thm:tb} are in the Appendix under Defs.~\ref{def:multi-traj-inflated-set}, \ref{def:multi-traj-inflated-tube} and Thm.~\ref{thm:multi-LDCs-tb}.

\looseness=-1


\looseness=-1
\noindent
\textbf{Action-based inflation.}
Action-based inflation is less direct than with trajectories: 
we inflate the neural network's \textit{output set} that is represented by a \textit{Taylor model} TM$(p(S_0), I)$~\cite{polar}, where $p(S_0)$ is a polynomial representing order-$k$ Taylor series expansion of the \cld activation functions in region $S_0$, and the remainder interval $I$ ensures that Taylor model overapproximates the neural network's output.  
In this context, we widen the bounds of the remainder interval $I$ in the last layer of the \cld by our statistical action-based discrepancy $\bar{\gamma}(D_0)$, ensuring that the potential outputs of \chd are contained in the resulting Taylor model. 


\begin{definition}[Action-inflated reachable set]
\label{def:action-inflate-oneldc}
        Given distribution $D_0$ over set $S_0$ that is controlled by LDC $\cld$, 
        statistical action-based discrepancy $ \bar{\gamma}(\DS )$, and low-dimensional control bounds  $[u_{min}(t), u_{max}(t)] \supseteq \cld\big(S_0\big)$, the \emph{action-inflated reachable set} contains states reachable by inflating the action bounds: 
        {\small
        \begin{align*}
            \irs(S_0, \bar{\gamma}(D_0)) &=  \big\{ f(s, u) \mid s \in S_0, u \in \big[u_{min}(t) - \bar{\gamma}(D_0), u_{max}(t) + \bar{\gamma}(D_0)\big] \big\}
        \end{align*}
        }
\end{definition}


\begin{definition} [Action-inflated reachable tube]
\label{def:action-inflate-tube}
    Given an distribution $D_0$ over initial set $S_0$ that is controlled by LDC \cld, dynamics $f$, time horizon $T$, and action-based discrepancy functions $ \bar{\gamma}(\DS )$, the \emph{action-inflated reachable tube} is a recursive sequence of inflated action-based reachable sets:$$ \irt(S_0, \bar{\gamma}(\DS ))  = \big[ S_0, \irs_1(S_0, \bar{\gamma}(\DS )), \irs_2(\irs_1, \bar{\gamma}(\DS )) , \dots , \irs_T(\irs_{T-1}, \bar{\gamma}(\DS )) \big]. $$
\end{definition}

Based on Defs. \ref{def:action-inflate-oneldc} and \ref{def:action-inflate-tube}, we put forward Thm.~\ref{thm:ab} below for the lower probability bound of the action-inflated LDC tube containing the true HDC tube.

\vspace{-2mm}
\begin{theorem}[Confident action-based overapproximation] \label{thm:ab} 
Consider distribution $D_0$ over initial set $S_0$, high-dimensional system \mhd with controller \chd, approximated by low-dimensional system $\mld$ controlled by $\cld$  with $\alpha$-confident statistical action-based discrepancies $\bar{\gamma}(S_0)$. Then the action-inflated low-dimen-\\sional tube $\irt_{\mld}(S_0, \bar{\gamma}(S_0))$ contains the high-dimensional tube $\rt_{\mhd}(S_0)$ with probability $1-\alpha$:
$$ \prob_{D_0} \Big[\rt_{\mhd}(S_0) \subseteq \irt_{\mld}(S_0, \bar{\gamma}(S_0)) \Big] \geq 1-\alpha $$ 
\end{theorem}

Defs.~\ref{def:action-inflate-oneldc} and \ref{def:action-inflate-tube} describe inflation with a \textit{single LDC}, which we extend to multiple LDCs by taking the union of all the LDCs' inflated tubes. Given a partitioned initial set $S_0 = \{S_1, ..., S_m \}$ with respective controllers $\{\cld^1, \dots, \cld^{m} \}$ and distributions $D_0 = \{D_1, ...,D_m \}$, the multiple LDCs action-inflated reachable tube is $\irt(S_0, \bar{\gamma}(\DS ))  := \bigcup_{i = 1}^{m}\irt(S_i, \bar{\gamma}(D_i))$. As it turns out, this reachable tube also contains the HDC tube with at least $1-\alpha$ chance (See Thm.~ \ref{thm:multi-ab} in the Appendix).

\vspace{-3mm}
\subsection*{Step 5: Iterative retraining and re-gridding}
\begin{algorithm}[t]
\caption{Iterative LDC training for the action-based approach}
\label{alg:iter-train}
\begin{algorithmic}
\Function{IterativeTrainingAB}{HDC $\chd$, image generator $g$, sample count $N$, initial state space $S_0$, confidence $\alpha$, discrepancy thresh. $\xi$, time steps $T$, goal set $G$}

\State {$\lambda, \epsilon \gets $ initial values}
\State $\mathbf{S} \gets $ initial gridding of $S_0: S_1, S_2, \dots$

\While{Computing resources last}

\For{$i = 1$ to $|\mathbf{S}|$}
\State $\cld^i \gets$ \textsc{TrainLDC}(\chd, g, $S_i$, $\lambda$, $\epsilon$) 
    \State $\delta^i \gets $ \textsc{ComputeActionDiscr}($\cld^i$, \chd, $g$, $S_i$, $\alpha$, $N$)    
    \If{$\delta^i > \xi$}
        \State $\epsilon \gets \epsilon/2$ 
        \Comment{Reduce MSE threshold}
    \EndIf
\EndFor

\If { $\hat{\delta} > \xi$ in some sub-region $\hat{\mathbf{S}} \subseteq \mathbf{S}$ }
\Comment{Too much discrepancy}

\State $\mathbf{S}' \gets \mathbf{S}$ with
refined re-gridding of $\hat{\mathbf{S}}$
\EndIf
\If{$\hat{\delta} \leq \xi \land \rs_{\mld}(\hat{\mathbf{S}}, T)  \not\subseteq G$ in some sub-region $\hat{\mathbf{S}} \subseteq  \mathbf{S}$}
\State  $\lambda \gets \lambda/2$ and keep the same $\epsilon$ in $\hat{\mathbf{S}}$
\Comment{Reduce Lipschitz threshold}
\EndIf
\State $\mathbf{S} \gets \mathbf{S}' $ \Comment{Use the updated grid}
\EndWhile
\State $\bar{\bm{\gamma}} \gets
{\delta^1, \delta^2, \dots}$
\State \textbf{return} $\cld^1, \cld^2, \dots, \bar{\bm{\gamma}} $
\EndFunction
\end{algorithmic}
\end{algorithm}

Once the inflated reachable tubes are obtained in Step 4, we focus on the regions of the current subset of the initial set where HDC simulations succeed --- yet safety verification fails. This can happen for two reasons: (i) overly high overapproximation error in the LDC reachability, or (ii) overly high conformal discrepancy bounds from $\bar{\beta}$ or $\bar{\gamma}$. We handle these issues as follows. 

\smallskip
\noindent
\textbf{Reducing reachability overapproximation error.} We lower the threshold for the Lipschitz constant $\lambda$ to retrain the respective LDCs in Step 1. In our experience, this almost always reduces the overapproximation in the LDC analysis and makes low-dimensional reachable tubes tighter --- but may result in higher statistical discrepancy bounds, addressed below.

\smallskip
\noindent
\textbf{Reducing conformal discrepancy bounds.} When these bounds are high, our LDC is not sufficient to imitate the HDC performance in a certain region of the state space. Given a desired discrepancy bound $\xi$, when this bound $\xi$ is exceeded in a state-space region, we split it into subregions by taking its midpoints in each dimension, leading to an updated state-space grid $\mathbf{S}'$. Then in each sub-region, we retrain an LDC as per Step 1 with a reduced acceptable MSE threshold $\epsilon$ and re-compute its bounds as per Step 3b. leading to tighter statistical overapproximations of HDC reachable tubes. 


\looseness=-1
To summarize, Alg.~\ref{alg:iter-train} shows our iterative training procedure for the action-based approach (its trajectory-based counterpart proceeds analogously, except for computing the discrepancies over trajectories, see Alg.~\ref{alg:iter-train-2} in the Appendix).


\vspace{-2mm}
\begin{algorithm}[tb]
\caption{End-to-end reachability verification of an HDC}
\label{alg:overall}
\begin{algorithmic}

\Function{EndToEndVerification}{HDC $\chd$,  generator $g$, sample count $N$, state space $S$, initial set $S_0$, confidence $\alpha$, discrepancy threshold $\xi$, time horizon $T$, goal set $G$, approach selection $J \in \{$ trajectory-based, action-based \}}


\If{$J =$ trajectory-based}
\State $\cld^1 \dots \cld^n, \bar{\beta} \gets$\textsc{IterativeTrainingTB}($\chd$, $g$, $N$, $S_0$,  $\alpha$,  $\xi$, $T$)
\State $X \gets \bar{\beta}$  \Comment{Store the trajectory discrepancies}

\Else 
\State $\cld^1 \dots \cld^n, \bar{\gamma} \gets$ \textsc{IterativeTrainingAB}($\chd, g, N, S, \alpha, \xi, T, G$)
\State $X \gets \bar{\gamma}$ \Comment{Store the action discrepancies}

\EndIf

\State $\mathbf{S_{ver}} \gets $ split $S_0$ into regions: $S_0^1, S_0^2, \dots$ \Comment{Gridding for parallel verification}

\State $S_{safe}, S_{unsafe} \gets \emptyset$ \Comment{Initialize safe and unsafe regions}

\For{$j = 1$ to $|\mathbf{S_{ver}}|$}
\State $
\irs(S_0^j, X, T) 
\gets \textsc{Reach} (\cld^1, \dots ,\cld^n, S_0^j, X, T)$ 
\If{
$\irs(S_0^j, X, T) \subseteq G$}
\State $S_{safe} \gets S_{safe} \cup S_0^j$
\Else
\State $S_{unsafe} \gets S_{unsafe} \cup S_0^j$
\EndIf

\EndFor

\State \textbf{return} $S_{safe}, S_{unsafe}$
\EndFunction
\end{algorithmic}
\end{algorithm}

\smallskip
Combining all the five steps together,
we present Alg.~\ref{alg:overall} that displays our end-to-end verification of a given HDC with either trajectory-based or action-based discrepancies. The LDCs and their discrepancies are input into the reachability analysis, implemented with the function \textsc{Reach}, to calculate the inflated reachable tubes (using the POLAR toolbox in practice). Our end-to-end algorithm guarantees that an affirmative answer to our verification problem is correct with at least $1-\alpha$ probability, as per Thm.~\ref{thm:safe}.  

\vspace{-1mm}
\begin{theorem}[Confident guarantee of HDC safety] \label{thm:safe} 
Consider a partitioned initial set grid $S_0 = \{S_1, \dots, S_{m}\}$, 
a set of corresponding distributions $\{D_1,...D_m \}$, 
a high-dimensional system \mhd with controller \chd, and a set of low-dimensional systems $\mld^1, \dots, \mld^m$ with respective controllers $\cld^1, \dots, \cld^n$ that approximate \chd with either an $\alpha$-confident trajectory discrepancy or action discrepancy, 
the probability that HDC safe set $S_{safe}$ calculated by Alg.~\ref{alg:overall} with either discrepancy belongs to ground truth safe set $ S^*_{safe}$  is at least $(1 - \alpha)$:
$$\prob_{D_1...D_m}\Big[ S_{safe} \subseteq S^*_{safe} \Big] \geq (1-\alpha) $$
\end{theorem}

\section{Experimental Evaluation} \label{sec:evaluation}
\vspace{-2mm}

\looseness=-1
\textbf{Benchmark systems and controllers.} We evaluate our approach on three benchmarks from OpenAI Gym~\cite{brockman_openai_2016}: two two-dimensional case studies --- an \textit{inverted pendulum} (IP) with angle $\theta$ and angular velocity $\dot{\theta}$; a \textit{mountain car} (MC) with position $x$ and velocity $v$, and a four-dimensional case study --- a \textit{cart pole} (CP) with cart position $x$, cart velocity $v$, angle $\theta$, and angular velocity $\dot{\theta}$. Our selection of case studies is limited because of the engineering challenge of setting up \textit{both} vision-based control and low-dimensional verification for the same system. Our continuous-action, convolutional HDCs \chd for these systems were trained with deep deterministic policy gradient (DDPG)~\cite{ddpg}. To imitate the performance of \chd, we train simpler feedforward neural networks \cld with only low-dimensional state inputs. See the Appendix for their architecture and dynamics, and our code can be accessed from GitHub \footnote{\url{https://github.com/yuanggeng/Bridging-dimensions}} 

\noindent
\textbf{Experimental procedure.} 
Our verification's goal is to check whether the system will stay inside the specified goal set $G$ after $T$ time steps (e.g., the mountain car's position must stay within the target set $[0.45, \infty ]$ after 60 steps). The verification returns ``safe'' if the inflated reachable set for $t=T$ lies entirely in $G$ --- and  ``unsafe'' otherwise. The details are found in the Appendix.

For both approaches, we calculate the discrepancies in 0.25-sized state squares within the initial set in IP, hence creating $8 \times 8 = 64$ regions (MC has $8 \times 9 = 72$ regions; CP has $5 \times 5 \times 5 \times 5 = 625$). In each, we sample 60 trajectories to compute both trajectory-based discrepancies $\bar{\beta}$ and action-based discrepancies $\bar{\gamma}$ because it is a relatively small sample count that avoids the highest non-conformity score or the infinity as the conformal bound. We also implement a \textit{pure conformal prediction baseline} and, for a fair comparison, give it the same data/regions. This results in 3840 sampled trajectories in IP, 4320 in MC, and 76800 for CP. 

\looseness=-1
We use closed-loop simulation to obtain the (approximate) ground truth (GT) of safety. For IP and CP, we grid the initial set into squares with an interval of 0.01. For MC, we grid the initial set with the position step 0.01 and velocity step 0.001. Within each grid cell, we uniformly sample 10 initial states and simulate a trajectory from each. If all 10 trajectories end in the goal set $G$, we mark this cell as ``truly safe'', otherwise ``truly unsafe''. In IP, the truly safe-to-unsafe cell ratio is 0.56, 0.78 in MC, and 0.58 in CP. 
The verification process uses the same grid cells as its initial state regions, leading to 40k low-dimensional verification runs for IP, 14k for MC, and 50k for CP. The trajectory-based verification time for IP, MC, and CP are 6.2, 5.8, and 6.4 hours respectively; the action-based verification takes 6.3, 6.1, and 6.6 hours respectively.

\looseness=-1
\noindent
\textbf{Success metrics.}
We evaluate verification as a binary classifier of the GT safety, with ``safe'' being the positive class and ``unsafe'' being the negative. 
Our evaluation metrics are the (i) \textit{true positive rate} (TPR, a.k.a. sensitivity and recall), indicating the fraction of truly safe regions that were successfully verified; (ii)  \textit{true negative rate} (TNR, a.k.a. specificity), indicating the fraction of truly unsafe regions that failed verification; (iii) \textit{precision}, indicating the fraction of safe verification verdicts that are truly safe (which is essential for safety-critical systems and controlled by rate $\alpha$ as per Thm.~\ref{thm:safe}); and (iv) \emph{F1 score}, which is a harmonic mean of precision and recall to provide a class-balanced assessment of predictions.

\begin{table*}[h!]
\centering
\vspace{-4mm}
\caption{Verification performance ($M=4$ for IP and CP, $M=10$ for MC).}
\vspace{-6mm}
\label{tab:environments}

\resizebox{\textwidth}{!}{%
\begin{tabular}{|c|c|c|C{2cm}|C{2cm}|C{2cm}|C{2cm}|}

\multicolumn{1}{c}{\multirow{4}{*}{Benchmark}} & \multicolumn{1}{c}{\multirow{4}{*}{Metrics}} & \multicolumn{5}{c}{} \\
\cline{1-7}
\multicolumn{1}{|c|}{} & \multicolumn{1}{c|}{} & Pure conformal prediction & \multicolumn{2}{c|}{Trajectory-based approach} & \multicolumn{2}{c|}{Action-based approach} \\
\cline{3-7}
\multicolumn{1}{|c|}{} & \multicolumn{1}{c|}{} & HDC & 1 LDC & $M$ LDCs & 1 LDC & $M$ LDCs\\
\hhline{|=======|}
\multirow{5}{*}{\shortstack{Inverted\\ Pendulum\\(IP)}} 

& True positive rate & 0.6564 &  0.4662 & \textbf{0.7938} & 0.0603 & 0.4050   \\ \cline{2-7}
& True negative rate & \textbf{0.9999} & 0.9976 & \textbf{0.9995}& \textbf{1.0000} & \textbf{0.9999}  \\ \cline{2-7}
& Precision & \textbf{0.9998} & 0.9880 & \textbf{0.9985} & \textbf{1.0000} & \textbf{0.9997}  \\ \cline{2-7}
& F1-score & 0.7925 & 0.6335 & \textbf{0.8844} & 0.1137 & 0.5765 \\ \hline
\multirow{5}{*}{\shortstack{Mountain\\ Car\\(MC)}} 

& True positive rate & 0.4686 &  \textbf{0.7220} & \textbf{0.7207} & 0.1050 & 0.2659  \\ \cline{2-7}
& True negative rate & \textbf{0.9967} & 0.9693 & 0.9872 & \textbf{0.9964} & \textbf{1.0000}  \\ \cline{2-7}
& Precision & \textbf{0.9916} & 0.9621 & 0.9793 & \textbf{0.9999} & \textbf{1.0000}   \\ \cline{2-7}
& F1-score & 0.6364 & 0.8249 & \textbf{0.8303} & 0.1900 & 0.4201 \\ \hline
\multirow{5}{*}{\shortstack{Cartpole\\(CP)}} 


& True positive rate & 0.6697 & 0.7225 & \textbf{0.7450} & 0.6554 & 0.7238 \\ \cline{2-7}
& True negative rate & \textbf{1.0000} & 0.9998 & \textbf{1.0000} & \textbf{1.0000} & \textbf{1.0000}  \\ \cline{2-7}
& Precision & \textbf{1.0000} & 0.9999 & \textbf{1.0000} & \textbf{1.0000} & \textbf{1.0000}\\ \cline{2-7}
& F1-score & 0.8022 & 0.8389 & \textbf{0.8539} & 0.7918 & 0.8398 \\ \hline
\end{tabular}%
}
\vspace{-4mm}
\end{table*}

\begin{table*}[h!]
\centering
\vspace{-9mm}
\caption{Verification performance for multiple LDCs with zero-mean Gaussian noise added to true state before image generator $g$.}
\vspace{-5mm}
\label{tab:noise}

{
{\fontsize{7}{9}\selectfont
\begin{tabular}{|c|c|C{1.5cm}|C{1.5cm}|C{1.5cm}|C{1.5cm}|}

\multicolumn{1}{c}{\multirow{3}{*}{Benchmark}} & \multicolumn{1}{c}{\multirow{3}{*}{Metrics}} & \multicolumn{4}{c}{} \\
\cline{1-6}
\multicolumn{1}{|c|}{} & \multicolumn{1}{c|}{} & \multicolumn{2}{c|}{Trajectory-based method} & \multicolumn{2}{c|}{Action-based method} \\
\cline{3-6}

\hhline{|======|}
\multirow{5}{*}{\shortstack{Inverted\\ Pendulum\\(IP)}} & \cellcolor{gray!25}{ STD of $\theta$, $\dot{\theta}$ noise } & \cellcolor{gray!25}0.01 & \cellcolor{gray!25}0.1 & \cellcolor{gray!25}0.01 & \cellcolor{gray!25}0.1 \\ \cline{2-6}
& True positive rate & \textbf{0.6732} & 0.5272 & 0.3675 & 0.1924  \\ \cline{2-6}
    & True negative rate & \textbf{1.0000} & \textbf{1.0000} & 0.9999 & \textbf{1.0000}  \\ \cline{2-6}
& Precision & \textbf{1.0000} & \textbf{1.0000} & 0.9997 & 0.9997  \\ \cline{2-6}
& F1-score & \textbf{0.8046} & 0.6904 & 0.5374 & 0.3228 \\ \hline
\multirow{6}{*}{\shortstack{Mountain\\ Car\\(MC)}} & \cellcolor{gray!25}STD of $x$ noise & \cellcolor{gray!25}0.01  & \cellcolor{gray!25}0.1 & \cellcolor{gray!25}0.01 & \cellcolor{gray!25}0.1 \\ \cline{2-6}
& \cellcolor{gray!25}STD of $v$ noise & \cellcolor{gray!25}0.0001 & \cellcolor{gray!25}0.003 & \cellcolor{gray!25}0.0001 & \cellcolor{gray!25}0.003 \\ \cline{2-6}
& True positive rate & \textbf{0.6797} & 0.4189 & 0.1558 & 0.0658  \\ \cline{2-6}
& True negative rate & 0.9878 & 0.9889 & \textbf{1.0000} & \textbf{1.0000}  \\ \cline{2-6}
& Precision & 0.9790 & 0.9753 & \textbf{1.0000} & \textbf{1.0000}  \\ \cline{2-6}
& F1-score & \textbf{0.8023} & 0.5861 & 0.2696 & 0.1235 \\ \hline
\multirow{5}{*}{\shortstack{Cartpole\\(CP)}} & \cellcolor{gray!25}STD of $x$,$v$,$\theta$,$\dot{\theta}$ noise & \cellcolor{gray!25}0.03 & \cellcolor{gray!25}0.1 & \cellcolor{gray!25}0.03 & \cellcolor{gray!25}0.1 \\ \cline{2-6}
& True positive rate & \textbf{0.7108} & 0.6253 & 0.6724 & 0.6040  \\ \cline{2-6}
& True negative rate & \textbf{0.9998} & \textbf{0.9998} & 0.9996 & 0.9996  \\ \cline{2-6}
& Precision & \textbf{0.9995} & \textbf{0.9995} & 0.9990 & 0.9989  \\ \cline{2-6}
& F1-score & \textbf{0.8308} & 0.7692 & 0.8038 & 0.7528 \\ \hline
\end{tabular}
}
}
\vspace{-5mm}
\end{table*}

\noindent
\textbf{Verification results.} The quantitative results of the three case studies are summarized in Tab.~\ref{tab:environments}. 
Confidence $\alpha$ is set to 0.05 for all methods, which sets the minimum precision to 0.95, satisfied by all the approaches. The pure conformal prediction baseline shows high precision and TNR, but loses in TPR to our approaches --- thus being able to correctly verify a significantly smaller region of the state space. When it comes to well-balanced safety prediction in practice, F1 score shows that our trajectory-based approach outperforms the other two. 

Across all case studies, the baseline is significantly more conservative than the requested 95\% precision. While this can be an advantage in safety-critical settings, excessive conservatism can also hamper adoption, so the approach should be sensitive to the desired confidence --- which our trajectory-based approach demonstrates in the mountain car case study (see Precision in Tab.~\ref{tab:environments}).

Across all case studies, the multi-LDC approaches always match or outperform the one-LDC approaches. This result demonstrates the utility of modularizing the HDC approximation problem. Also, our single-LDC action-based approach successfully verifies relatively few regions, leading to its low TPR. That is because unlike in the case of trajectory discrepancies, only one LDC cannot provide tight statistical upper bounds for control actions, causing large overapproximation in the inflated reachable sets, resulting in false negatives.


\smallskip
\looseness=-1
\noindent
\textbf{Sensitivity to noisy images.} Despite adding Gaussian noise to generator $g$, our approaches perform similarly to noise-free $g$ when under low noise variance as per Tab.~\ref{tab:noise}, thus showing some robustness. 
However, we saw a significant decline in the verification coverage (TPR, but not the TNR and $\alpha$-guaranteed precision) under substantial noise variance (up to 0.5, not shown in Tab.~\ref{tab:noise}).

\smallskip
\looseness=-1
\noindent
\textbf{Limitations.} 
Our approach relies on statistical inference based on i.i.d. sampling from a fixed distribution, which downgrades the exhaustive guarantees of formal verification. However, it may be possible to exhaustively bridge this gap with neural-network conformance analysis based on satisfiability solving~\cite{mohammadinejad_diffrnn_2021}. 
We also envision relaxing the i.i.d. assumption with time-series conformal prediction~\cite{xu_conformal_2021,auer_conformal_2023}, as well as uncertainty-guided gridding~\cite{CP2} to reduce our discrepancy bounds.




\vspace{-3mm}
\section{Related work}\label{sec:relwork}


\vspace{-3mm}
\textbf{Low-dimensional verification of closed-loop systems.}
Neural-network controlled systems have been used widely~\cite{imitatelearning,ruchkin_confidence_2022,topcu_assured_2020}, which has highlighted the challenges of verifying their correctness within closed-loop systems. Since it's impossible to calculate all the exact states, especially in non-linear systems, current approaches primarily focus on how to make tight overapproximate reachable sets~\cite{flow,RAsurvay,CORA}. For sigmoid-based NNCS, Verisig~\cite{verisig} toolbox can transform the neural-network controlled system into a hybrid system, which can be verified by other tools like flow*. NNV~\cite{nnv} performs overapproximation analysis by combining star sets~\cite{star-based,star2} for feed-forward neural networks with zonotopes for non-linear plant dynamics in CORA~\cite{CORA}. POLAR~\cite{polar} overcame the challenges of non-differentiable activation functions by combining the Bernstein-B\'{e}zier Form~\cite{reachnn} and the symbolic remainder. This method achieves state-of-the-art performance in both the tightness of reachable tubes and computation times. Another type of verification called \textit{Hamilton-Jacobi} (HJ) reachability~\cite{hj-overview}, is inspired by optimal control.  
The DeepReach~\cite{deepreach} technique can solve the verification problem with tens of dimensions by leveraging a deep neural network to represent the value function in the HJ reachability analysis. Nonetheless, such methods remain ill-suited for handling inputs with hundreds or thousands of dimensions.

These verification tools cannot deal with complicated neural network controllers. Therefore, an alternative approach is to simplify complex controllers into smaller, verifiable controllers by model reduction techniques~\cite{surveyMR,model_reduct_new}, such as parameter pruning, compact convolution filters, and knowledge distillation~\cite{distill}.

\noindent
\looseness=-1
\textbf{Abstractions of perception models.} Given the challenge of verifying the image-based closed-loop systems directly, many methods construct abstractions of the perception model to map the relationship between the image and the states for verification~\cite{closed_analysis_vision}. One abstraction approach~\cite{gan} employs the generative model, especially Generative Adversarial Network (GAN), mapping states to images. The generated images will be put into the controller in the verification phase. Hence, the accuracy of the verification results depends on the quality of the image produced by the generative model. Other researchers~\cite{AAPs} construct the exact mathematical formula mapping the real state into the simplified image~\cite{NNVerifier}, which can be verified in another neural network checker~\cite{AAPschecker}. One limitation of exact modeling is the effort to generalize for other systems or scenarios. For instance, their implementation may be specific to a proportional controller in the aircraft landing or lane-keeping scenarios, which may not be suitable for the more complicated image-based systems in other cases. 

\looseness=-1

\noindent
\textbf{Statistical verification}. Statistical verification draws samples to determine the property satisfaction from a finite number of trajectories~\cite{SMCsurvay,smc2,lew_simple_2022,cleaveland-traj}. One advantage of such algorithms is that they provide assurance for arbitrarily complex black-box systems, merely requiring the ability to simulate them~\cite{PAC,Sample_wang}. 
Conformal prediction~\cite{vovk_algorithmic_2005}, which has been a popular choice for distribution-free uncertainty quantification, has recently been used to provide probabilistic guarantees on the satisfaction of a given STL property~\cite{CP,CP2}. Purely statistical methods come at the price of drawing sufficient samples --- and only obtaining the guarantees at some level of statistical confidence, which can be difficult to interpret in the context of a dynamical system. Our work restricts the use of sampling only to the most challenging aspects and leverages exhaustive verification for the rest of the system, thus reducing our reliance on statistical assurance.

\vspace{-4mm}
\section{Conclusion}\label{sec:conclusion}
\vspace{-3mm}
\looseness=-1
This paper takes a significant step towards addressing the major challenge of verifying end-to-end controllers implemented with high-dimensional neural networks. Our insight is that the behavior of such neural networks can be effectively approximated by several low-dimensional neural networks operating over physically meaningful space. To balance approximation error and verifiability in our low-dimensional controllers, we harness the state-of-the-art knowledge distillation. To close the gap between low- and high-dimensional controllers, we apply conformal prediction and provide a statistical upper bound on their difference either in trajectories or actions. Finally, by inflating the reachable tubes with two discrepancy types, we establish a high-confidence reachability guarantee for high-dimensional controllers. Future work may further reduce the role of sampling.

\looseness=-1


\bibliographystyle{splncs04}


\newpage
\section{Appendix}

Our appendix is organized into five parts: 
\begin{itemize}
    \item Subsection \ref{sec:sub_def} describes complementary definitions, which supplement those presented in the main body. 
    \item Subsection \ref{sec:sub_proof} provides all the proofs for our theorems. 
    \item Subsection \ref{sec:sub_quality} provides qualitative results. 
    \item Subsection \ref{sec:sub_case} shows the details of our case study. 
    \item Subsection \ref{sec:sub_algo} includes all complementary algorithms related to this paper.
\end{itemize}

\subsection{Complementary definitions}\label{sec:sub_def}

This section starts with widely accepted concepts to provide an exhaustive formalization behind Secs.~\ref{sec:background} and ~\ref{sec:approach}. 

\begin{definition}[State function]
Given initial state $s_0$, the state of (either low- or high-dimensional) system $M$ at discrete time $t$ is  $\varphi_M(s_0, t)$, given by \emph{state function} $\varphi_M:S \times \mathbb{N}_{\geq 0} \rightarrow S$.
\end{definition}

\begin{definition}[Trajectory for \mld and \mhd]
\label{def:trajectory}
    From state $s_0$,  system \mld produces the \emph{low-dimensional trajectory} $\tld$ for $T$ steps: $\tld(s_0, T)=[s_0, \varphi_{\mld}(s_0, 1),$ $\dots,  \varphi_{\mld}(s_0, T)].$
    From state $s_0$, system \mhd produces the \emph{high-dimensional trajectory} $\thd$ for $T$ steps: $\thd(s_0, T) = [s_0,  \varphi_{\mhd}(s_0, 1), \dots,  \varphi_{\mhd}(s_0, T)].$
\end{definition}

\begin{definition}[Lipschitz Constant for LDC]
\label{def:lp_constant}
    An LDC \(\cld : S \rightarrow U\) is Lipschitz-continuous with \emph{Lipschitz constant} \(L\) on state region $\bar{S}$ if there exists \(L\geq 0\) such that:
$\forall s_1, s_2 \in \bar{S} \cdot \|\cld(s_1) - \cld(s_2)\|_2 \leq L \|s_1 - s_2\|_2$.
\end{definition}

Here we formalize the two notions of discrepancy. 

\begin{definition}[Trajectory-based discrepancy]
\label{def:real_trajectory_dis}
A \emph{trajectory-based discrepancy} $\beta(S_0)$ is the supremum of the L1 differences between high- and low-dimensional trajectories starting in initial set $S_0$ at time $t \in [0 .. T]$:  
  $$\beta(S_0) := \sup_{\forall s_0 \in S_0, t \in [0..T]}
 \|\thd(s_0,t) - \tld(s_0,t)\|_1.
 $$
\end{definition}

\begin{definition}[Action-based discrepancy]
\label{def:real_action_dis}
An \emph{action-based discrepancy} $\gamma(S_0)$ is the supremum of the control action differences $\delta$ between an HDC and LDC on the initial set $S_0$ for each time until horizon $T$: 
        $$ \gamma(S_0) := \sup_{\forall s_0 \in S_0, t \in [0..T]}
  \|\chd\big(g(\varphi_{\mhd}(s_0, t))\big) - \cld\big(\varphi_{\mld}(s_0, t)\big)\|_1.$$
\end{definition}

The definitions and theorems below straightforwardly extend Defs.~\ref{def:traj-inflated-set}, \ref{def:traj-inflated-tube}, \ref{def:action-inflate-oneldc}, \ref{def:action-inflate-tube} and Thms. \ref{thm:tb}, \ref{thm:ab} to the multi-LDC case.



\begin{definition}[Multi-LDC Trajectory-inflated reachable set] \label{def:multi-traj-inflated-set}
Given a set of distribution $\mathbf{D_0} = \{D_1,...,D_m \}$ over partitioned initial set $\mathbf{S_0} = \{S_1,..., S_m \}$ that is controlled by $\{\cld^1, \dots, \cld^{m} \}$, reachable set $\rs(S_i, t)$, and its trajectory discrepancy $\bar{\bm{\beta}} = \{ \bar{\beta}(D_1),..., \bar{\beta}(D_m)\}$, a \emph{Multi-LDC trajectory-inflated reachable set} is defined as: 
        \begin{align*}
        \irs(\mathbf{S_0}, t, \bar{\bm{\beta}} ) = \bigcup_{i=1}^{m} \irs(S_i, t, \bar{\beta}(D_i)) = \bigcup_{i=1}^{m} \big\{s \in S ~|~ \exists s' \in \rs(S_i, t) \cdot     
        \|s-s'\|_1 \leq \bar{\beta}(D_i)\big\}.
        \end{align*}

\end{definition}

\begin{definition} [Multi-LDC Trajectory-inflated reachable tube] \label{def:multi-traj-inflated-tube}
Given a set of distribution $\mathbf{D_0} = \{D_1,...,D_m \}$ over partitioned initial set $\mathbf{S_0} = \{S_1,..., S_m \}$ that is controlled by $\{\cld^1,\cld^2, \dots, \cld^{m} \}$, reachable set $\rs(S_i, t)$, and its trajectory discrepancy $\bar{\bm{\beta}} = \{\bar{\beta}(D_1),..., \bar{\beta}(D_m)\}$, a \emph{Multi-LDC trajectory-inflated reachable tube} is defined as: 

\begin{align*}
    \irt(\mathbf{S_0}, \bar{\bm{\beta}}) &= \bigcup_{i=1}^{m} \irt(S_i, \bar{\beta}(D_i)) \\
    &= \bigcup_{i=1}^{m} \left[\irs(S_i, 0, \bar{\beta}(D_i)), \irs(S_0, 1, \bar{\beta}(D_i)), \dots, \irs(S_i, T, \bar{\beta}(D_i))\right].
\end{align*}
\end{definition}

\begin{theorem}[Confident multi-LDCs trajectory-based overapproximation]\label{thm:multi-LDCs-tb} 
Consider a set of distribution $\mathbf{D_0} = \{D_1,...,D_m \}$ over split initial set $\mathbf{S_0} = \{S_1,..., S_m \}$, 
a high-dimensional system \mhd, and a set of low-dimensional systems with a set of controllers $\cld = \{\cld^1, \dots, \cld^m \}$ that approximates it with an $\alpha$-confident statistical trajectory-based discrepancy function  $\bm{\bar{\beta}} = \big\{\bar{\beta}(D_1),... \bar{\beta}(D_m)  \big\}$. Then the trajectory-inflated low-dimensional tube $\irt_{\mld}(\mathbf{S_0}, \bar{\beta}(D_0))$ contains the high-dimensional reachable tube $\rt_{\mhd}(S_0)$ with probability $1-\alpha$ over the time $T$: 
$$ \prob_{D_1...D_m}\Big[\rt_{\mhd}(S_0) \subseteq \irt_{\mld}(\mathbf{S_0}, \bm{\bar{\beta}}) \Big] \geq 1-\alpha $$ 
\end{theorem}

\begin{definition} [Multi-LDC action-inflated reachable set]\label{def:action-inflate-manyldc}
Given distributions $\{D_1, ..., D_m\}$ over each region of grid $\mathbf{S_0} = \{S_1, \dots, S_{m}\}$ of initial set $S_0$, an LDC for each region $\{\cld^1,\cld^2, \dots, \cld^{m} \}$, low-dimensional reachable sets $\rs(S_i, t)$, dynamics $f$, and action-based discrepancy functions for each LDC and state region  $\bar{\bm{\gamma}} = \big\{\bar{\gamma}(D_i) \big\}_{i=1}^{m}$, low-dimensional control bounds $[u_{min}^i(t), u_{max}^i(t)] \supseteq \cld^i\big(\rs(S_i, t)\big)$, the \emph{multi-LDC action-inflated reachable set} contains a union of states reachable by inflating the actions bounds of the individual controllers:
        {\small
        \begin{align*}
            \irs(\mathbf{S_0}, t+1, 
            \bar{\bm{\gamma}}) &= \bigcup_{i=1}^{m} \irs\big(S_i, t, \bar{\gamma}(\Di\big)  \\
            = \bigcup_{i=1}^{m} \Big\{ f(s_i, u_i) \mid ~& s_i \in \rs(S_i, t), u_i \in \big[u_{min}^i(t) - \bar{\gamma}(D_i), u_{max}^i(t) + \bar{\gamma}(D_i)\big]\Big\}
        \end{align*}
        }

\end{definition}
\begin{definition} [Multi-LDC action-inflated reachable tube]
    Given an initial set $S_0$, an initial state-space grid $\mathbf{S_0} = \{S_1, \dots,S_{m}\}$, multiple LDCs $\{\cld^1,\cld^2, \dots, \cld^{m} \}$, dynamics $f$, and action-based discrepancy functions for each LDC and state region $\bar{\bm{\gamma}} = \big\{\bar{\gamma}(S_i) \big\}_{i=1}^{m}$, the \emph{multi-LDC action-inflated reachable tube} is a recursive sequence of inflated action-based reachable sets:$$ \irt(S_0, \bar{\bm{\gamma}})  = \big[ S_0, \irs_1(S_0, S_0, \bar{\bm{\gamma}}), \irs_2(S_0, \irs_1, \bar{\bm{\gamma}}) , \dots , \irs_T(S_0, \irs_{T-1}, \bar{\bm{\gamma}}) \big]. $$
\end{definition}

\begin{theorem}[Confident multi-LDCs action-based overapproximation] \label{thm:multi-ab} 
Consider an initial set $S_0$ that can be partitioned into $m$ subsets, $\mathbf{S_0} = \{ S_1,..., S_m \}$, a high-dim. system \mhd with controller \chd, and low-dimensional systems $\mld^1, \dots, \mld^m$ with respective controllers $\cld^1, \dots, \cld^m$ that approximate \chd with $\alpha$-confident statistical action-based discrepancies $\bm{\bar{\gamma}} = \big\{\bar{\gamma}(\Di) \big\}_{i=1}^{m}$. Then the action-inflated low-dimensional tube $\irt_{\mld}(S_0, \bar{\bm{\gamma}})$ contains the high-dimensional tube $\rt_{\mhd}(S_0)$ with probability $1-\alpha$ over the time $T$: 
$$ \prob_{D_1\dots D_m} \Big[\rt_{\mhd}(S_0) \subseteq \irt_{\mld}(\mathbf{S_0}, \bar{\bm{\gamma}}) \Big] \geq 1-\alpha $$ 
\end{theorem}

\subsection{Proofs of theorems}\label{sec:sub_proof}

This section contains theorem proofs that could not be included in the main body of the paper due to the page limit. 

\subsubsection*{Proof of Theorem 1 (Confident trajectory-based overapproximation)}

\begin{proof}
Based on Lemma 1, after calculating the conformal bound $\bar{\beta}(D_0)$ (in practice, $D_0 = \operatorname{Uniform}(S_0)$) with specified miscoverage $\alpha$ for the statistical trajectory-based discrepancy, for any new initial point $s_0$  sample uniformly and i.i.d. from $S_0$, the following holds: 
\begin{align*}
  \forall t \in [0..T] \cdot &\prob_{D_0}\Big[ \|\varphi_{hd}(s_0,t) - \varphi_{ld}(s_0,t)\|_1 \leq \bar{\beta}(D_0)\Big] \geq 1 - \alpha
\end{align*}

The low-dimensional reachable set $\rs(S_0,t)$ is guaranteed to contain all trajectories starting from the initial set $S_0$ in system $\mld$; The $\bar{\beta}(D_0)$ is the confidence upper bound of the maximum L1 norm over the whole trajectories starting from the initial set. Then, with at least $1-\alpha$ probability,
the statistical lower bound of the trajectory-inflated reachable set is: 
$$\irs_{low}(S_0, t) = \rs_{low}(S_0, t) - \bar{\beta}(D_0)$$ 
Similarly, the statistical upper bound of the trajectory-inflated reachable set is:$$\irs_{up}(S_0, t) = \rs_{up}(S_0, t) + \bar{\beta}(D_0)$$


Then if any trajectory differences starting from this initial set satisfy the specified confidence, the reachable set containing all trajectories will satisfy this confidence, too. Since the whole reachable tube contains all the reachable sets, this inflated reachable tube also satisfies this probability bound:
$$ \prob_{D_0}\Big[\rt_{\mhd}(S_0) \subseteq \irt_{\mld}(S_0, \bar{\beta}(D_0)) \Big] \geq 1-\alpha.$$
\end{proof}

\subsubsection*{Proof of Theorem 2 (Confident action-based overapproximation)}

\begin{proof}
Consider a low-dimensional reachable tube inflated with action-based discrepancy bounds. Given the statistical action-based upper bound $\bar{\gamma} = \bar{\gamma}(D_0)$ (in practice, $D_0 = \operatorname{Uniform}(S_0)$) with one LDC over the initial set $S_0$, starting from any states in $S_0$ sample uniformly, in the next $T$ time steps, control action difference between LDC and HDC is within this conformal bound with at least $1-\alpha$ confidence,
$$ \prob_{s_0 \sim \operatorname{Uniform}(S_0)}\Big[\max_{t=0..T} \|\chd\big(g(\varphi_{\mhd}(s_0, t))\big) - \cld\big(\varphi_{\mld}(s_0, t)\big)\|_1 \leq \bar{\gamma}\Big] \geq 1 - \alpha.$$

Therefore, the $(1-\alpha)$-confident statistical upper and lower bounds on the \chd output are respectively $[u_{min}^{\cld} - \bar{\gamma}(D_0), u_{max}^{\cld} + \bar{\gamma}(D_0)]$ for each time step in the whole time horizon $T$, where $u_{min}^{\cld}$ and $u_{max}^{\cld}$ are the bounds on \cld output based on the reachability analysis in Step 2 and the $\bgamma(D_0)$ has been calculated. Through the reachability analysis, we calculated each inflated $\gamma$-based reachable set given inflated control bounds in each time step. We assert that,
$$ \prob_{D_0} \Big[\forall t \in [0,T], \rs_{\mhd}(S_0, t) \subseteq \irs_{\mld}(S_0, t, \bar{\gamma}(D_0)) \Big] \geq 1-\alpha $$ 

Since the action-inflated reachable tube contains all the action-inflated reachable sets, the reachable tube of \chd also obtains the probabilistic containment in the inflated reachable tube, namely: 
$$ \prob\Big[\rt_{\mhd}(S_0) \subseteq \irt_{\mld}(S_0, \bar{\gamma}(D_0)) \Big] \geq 1-\alpha. $$ 

\end{proof}

\subsubsection*{Proof of Theorem 3 (Confident guarantee of HDC safety)} 
\begin{proof}
With the help of Thms.~\ref{thm:tb} and ~\ref{thm:ab}, given all the trained $\cld$ and their corresponding $\alpha$-confident statistical action-based discrepancy functions $\bar{\gamma} = \{\bar{\gamma}(S_i) \}_{i=1}^{n}$ or $\alpha$-confident statistical trajectory-based discrepancy function $\bar{\beta} = \{\bar{\beta}(S_i) \}_{i=1}^{n}$, both action-based and trajectory-based inflated reachable tubes contain the ground truth HDC reachable tube with a lower probability bound:
$$ \prob \Big[\rt_{\mhd}(S_0) \subseteq \irt_{\mld}(S_0, \bar{\beta}) \Big] \geq 1-\alpha. $$ 
$$
\prob \Big[\rt_{\mhd}(S_0) \subseteq \irt_{\mld}(S_0, \bar{\gamma}) \Big] \geq 1-\alpha. $$

We check the safety of the HDC by checking that $ \irs(S_0^j, T) \subseteq G)$. Since the target set and goal set are fixed, the confidence of safety depends on the confidence of the reachable tube and set starting from partitioned initial regions $\mathbf{S_{ver}} = {S_1, \cdots, S_n}$. 
Since the verification regions are independent, the probability that the HDC safety set $S_{safe}$ being contained in the ground truth safety set $S^*_{safe}$ is at least $(1-\alpha)$:
$$\prob\Big[ S_{safe} \subseteq S^*_{safe} \Big] \geq 1-\alpha $$

\end{proof}

\subsubsection*{Proof of Theorem 4 (Confident multi-LDCs trajectory-based overapproximation)}

\begin{proof}
This proof is based on Thm. \ref{thm:tb}.
Given the statistical action-based upper bound $\bm{\bar{\beta}} = \{\bar{\beta}(D_i)\}^m_{i=1}$ for LDCs ($D_i = \operatorname{Uniform}(S_i)$) over the partitioned initial set $\mathbf{S_0} = \{S_i\}_{i=1}^m$, for each subset $S_i$, Thm. \ref{thm:tb} gives us that 

$$ \prob\Big[\rt_{\mhd}(S_i) \subseteq \irt_{\mld}(S_i, \bar{\beta}(D_i)) \Big] \geq 1-\alpha. $$ 
Then we iterate each subset of $S_0$ to get the probability containment,
$$ \prob\Big[\bigcup_{i=1}^{m} \big( \rt_{\mhd}(S_i) \subseteq \irt_{\mld}(S_i, \bar{\beta}(D_i)) \big) \Big] \geq 1-\alpha. $$ 

Therefore, we add all the subregions together and conclude that 
$$ \prob_{D_1\dots D_m} \Big[\rt_{\mhd}(S_0) \subseteq \irt_{\mld}(\mathbf{S_0}, \bar{\bm{\beta}}) \Big] \geq 1-\alpha $$ 
\end{proof}

\subsubsection*{Proof of Theorem 5 (Confident multi-LDCs action-based overapproximation)}

\begin{proof}
This proof is based on Thm. \ref{thm:ab}. 
Given the statistical action-based upper bound $\bm{\bar{\gamma}} = \{\bar{\gamma}(D_i)\}^m_{i=1}$ for LDCs ($D_i = \operatorname{Uniform}(S_i)$) over the initial set $S_0 = \{S_i\}_{i=1}^m$, for each subset $S_i$, Thm. \ref{thm:ab} gives us that 

$$ \prob\Big[\rt_{\mhd}(S_i) \subseteq \irt_{\mld}(S_i, \bar{\gamma}(D_i)) \Big] \geq 1-\alpha. $$ 
And then we iterate each subset of $S_0$ to get the probability containment,
$$ \prob\Big[\bigcup_{i=1}^{m} \big( \rt_{\mhd}(S_i) \subseteq \irt_{\mld}(S_i, \bar{\gamma}(D_i)) \big) \Big] \geq 1-\alpha. $$ 

Therefore, we add all the subregions together, we can conclude that 
$$ \prob_{D_1\dots D_m} \Big[\rt_{\mhd}(S_0) \subseteq \irt_{\mld}(\mathbf{S_0}, \bar{\bm{\gamma}}) \Big] \geq 1-\alpha $$ 
\end{proof}

\subsection{Complementary qualitative illustration}\label{sec:sub_quality}

For most truly safe partitions of the initial state set, both of our approaches successfully verified safety. A successful example of the action-based approach is shown on the left of Fig.~\ref{fig:one_case}.  Given the narrow margins of safety in our benchmarks, even small overapproximation or discrepancy errors can lead to losing the safety guarantees, resulting in a false negative verification outcome. Such an example is shown on the right of Fig.~\ref{fig:one_case}: the reach set is not fully contained in the goal set (even though all simulations from the initial set are contained and therefore safe). 

These figures display only action-inflated reachable tubes: low-dimensional tubes and trajectory-inflated tubes are visually indistinguishable due to the low statistical discrepancies between the HDC and trained LDCs. Our reachable tube computations enable solving reach-avoid problems, which specify an avoid set in addition to a reach set. 

\begin{figure}[ht]
    
    \begin{subfigure}
        \centering
    \includegraphics[width=0.48\linewidth]{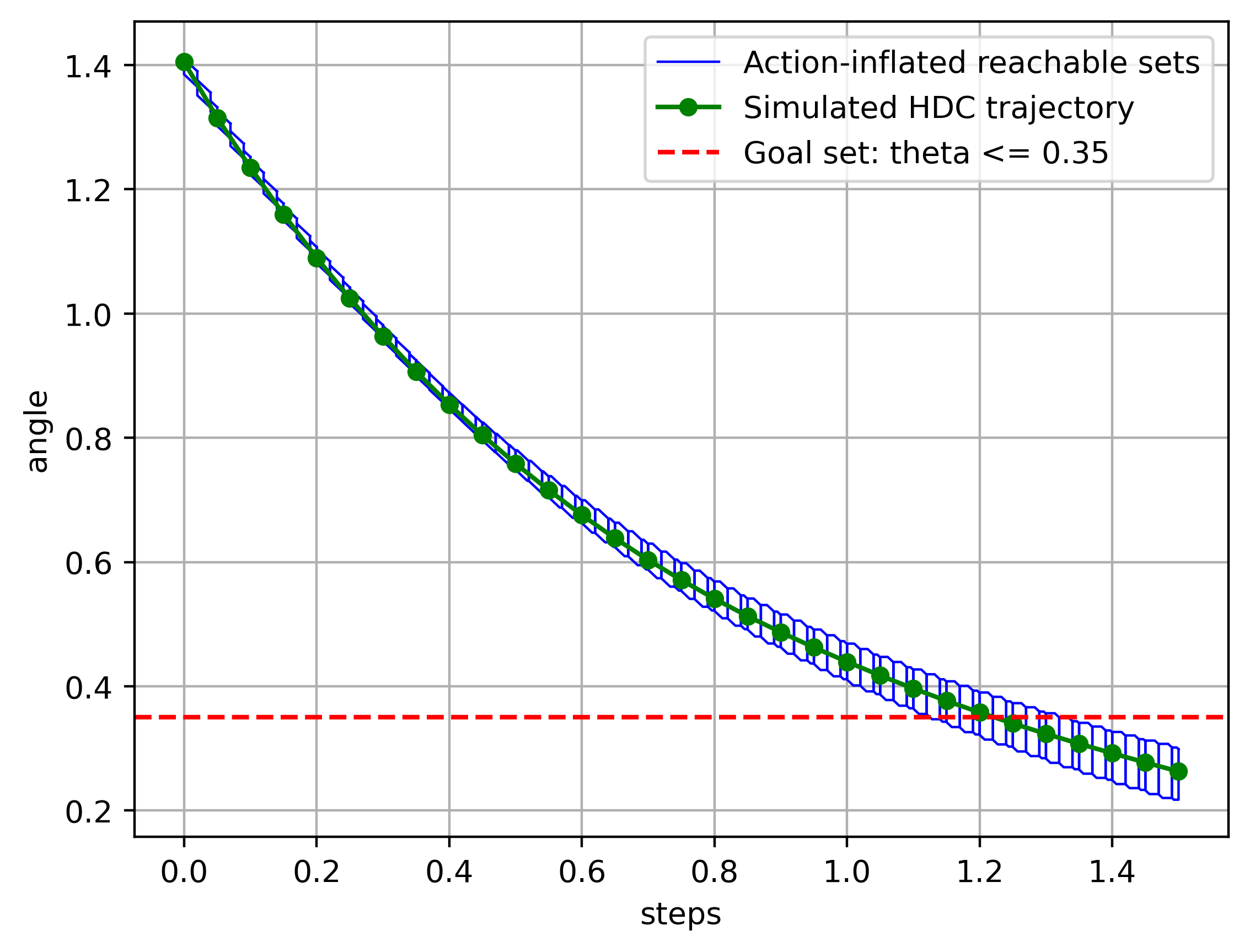}
        \label{fig:success_verify}
    \end{subfigure}
    \begin{subfigure}
        \centering
        \includegraphics[width=0.48\linewidth]{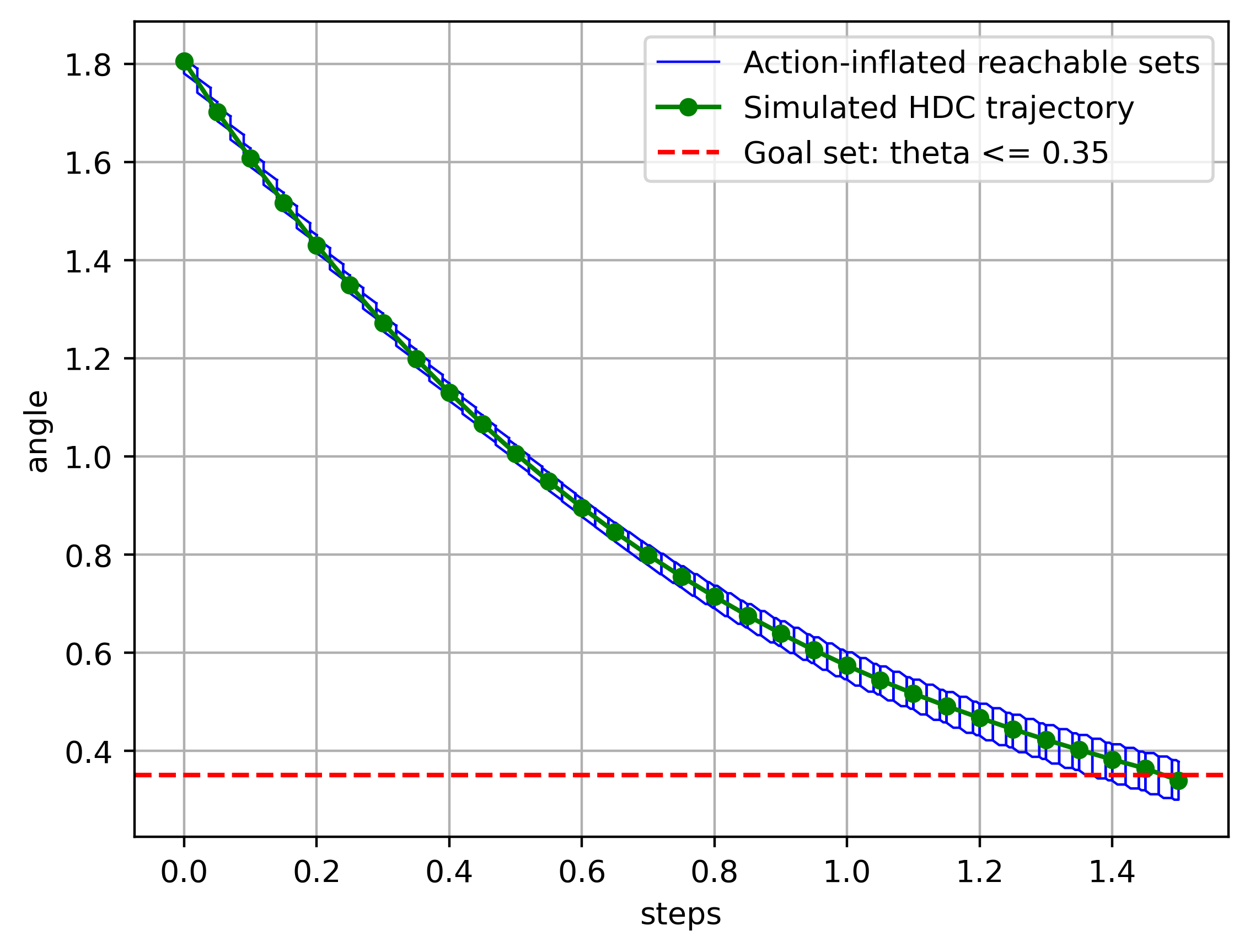}
        \label{fig:failed_verify}
    \end{subfigure} 
    \vspace{-6mm}
    \caption{Reachability examples of the action-based approach on the inverted pendulum. Left: safe, right: unsafe.}
    \label{fig:one_case}
\end{figure}


\begin{figure}[th]
\centerline{\includegraphics[width=0.9\columnwidth]{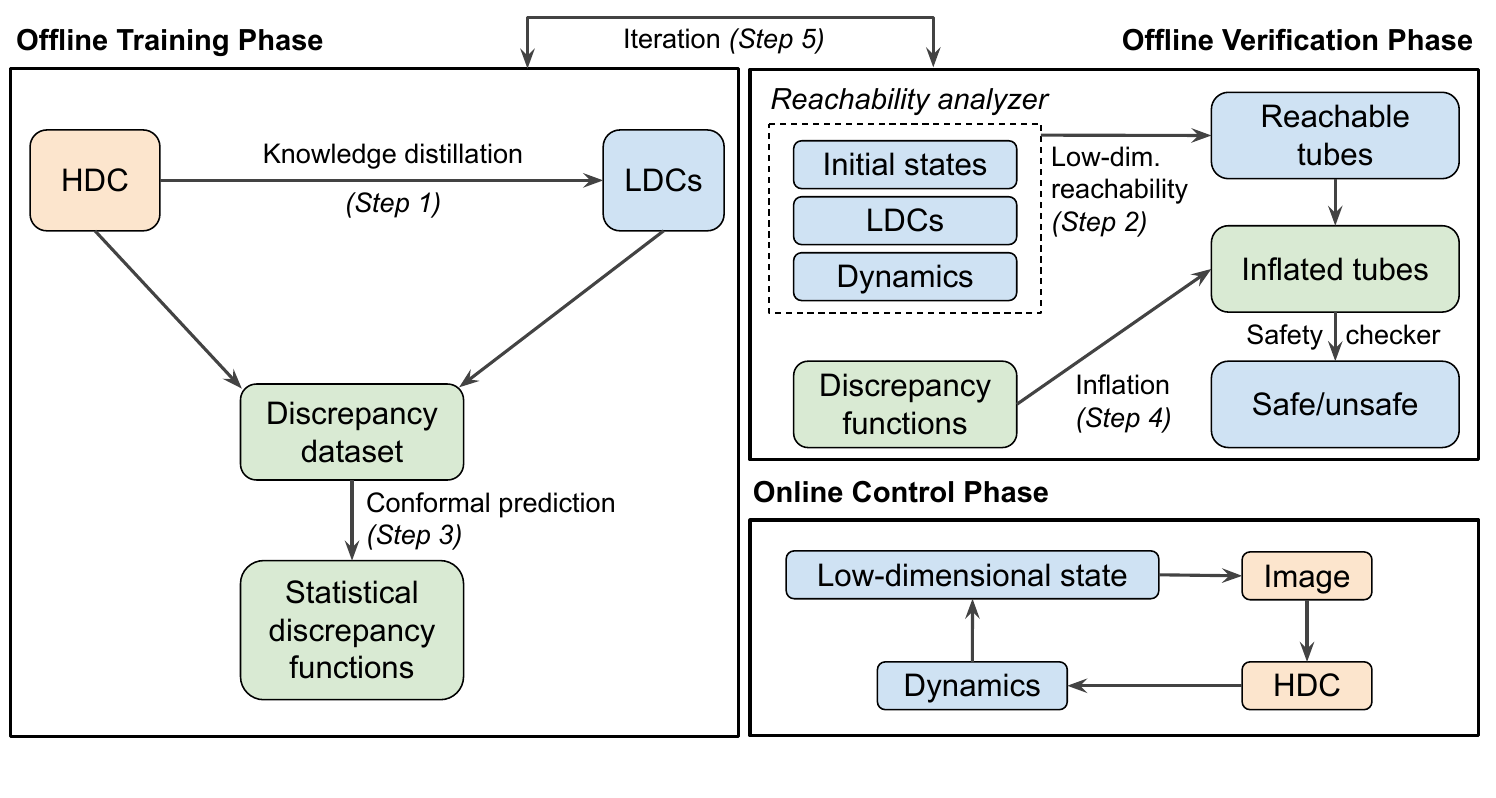}}
\vspace{-4mm}
\caption{Three phases of our approach: training, verification, and control. Orange shows high-dimensional elements, and green shows novel contributions.}
\label{fig:structure}
\end{figure}

\begin{figure}
\centerline{\includegraphics[width=0.8\columnwidth]{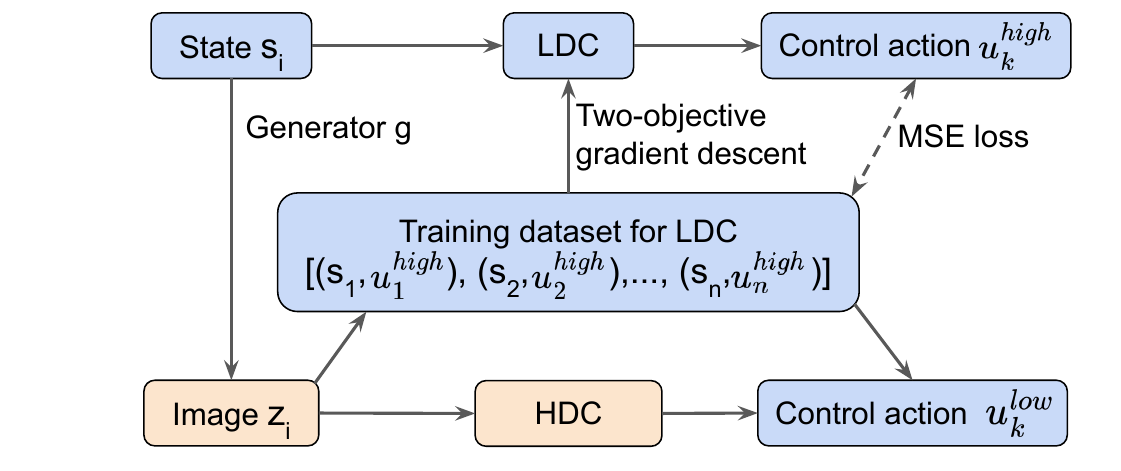}}
\vspace{-3mm}
 \caption{Training an LDC with supervision from an HDC. 
}
\label{fig:ldc-training}
\end{figure}

\subsection{Details of experimental setup}\label{sec:sub_case}


\vspace{-1mm}
\subsubsection*{Approach and initial set}

We present a detailed summary of our approach, illustrated step-by-step in Fig.~\ref{fig:structure}. Our methodology begins with matching the initial state and image, where the initial set \( S_0 \) can be construed as a quasi-inverse set \( g^{-1}(z_0) \) of an initial image \( z_0 \). This initial set \( S_0 \) is readily available to us due to our assumption that the initial conditions are sufficiently instrumented, in contrast to the subsequent states. Moreover, our approach is not confined to HDCs with solely high-dimensional inputs. For instance, in all of our case studies, the HDCs accommodate velocities alongside the images. Furthermore, our approach could be extended to HDCs that process sequences of images (e.g., to extract velocity information).

\subsubsection*{Benchmark systems} 
The equations below describe the discrete-time, friction-free, transition dynamics of the systems. We have omitted the time step $\Delta t$ for simplicity, assuming a constant time step, where the effects of the changes in all relevant variables are implicitly assumed to occur over a fixed interval.

The \ul{inverted pendulum dynamics} are as follows with the fixed parameters --- the mass of the rod ($m = 0.1 kg$), rod inertia ($I = 0.125 kg \cdot {m}^2$), rod length ($L = 0.25 m$), and damping ratio ($c = 0 
 {kg}^{-1} \cdot {m}^{-2}$): 
\vspace{-2mm}
\begin{align}
    \theta_{t+1} = \theta_t + \dot{\theta}_t, \qquad 
    \dot{\theta}_{t+1} = \dot{\theta}_{t} +  \frac{1}{I} (u_t - c\dot{\theta}_t + mgL\sin(\theta_t))
\end{align}
\vspace{-3mm}

The \ul{mountain car dynamics} are as follows with the fixed parameters --- the mass of the car ($m = 2.5 \times 10^{-4} kg$) and the force produced by the car's engine ($F = 3.75 \times 10^{-7} N$): 
\vspace{-1mm}
\begin{align}
    x_{t+1} = x_t + v_t , \qquad  v_{t+1}  = v_{t} +  \frac{F}{m} u_t - mg\cos(3x_t)
\end{align}
\vspace{-2mm}

The \ul{cartpole dynamics} are as follows with the fixed parameters --- the mass of the cart ($m_c = 1 kg$), the mass of the pole ($m_p =0.1 kg$) and the length of the pole ($l =0.5 m$): 
\vspace{-1mm}
\begin{align}
    x_{t+1} = x_t + v_t , \qquad \dot{x}_{t+1} = \dot{x}_t + \frac{u + m_pl(\dot{\theta}^2\sin{\theta}-\Ddot{\theta}\cos{\theta})}{m_c+m_p} \\ \theta_{t+1} = \theta_t + \dot{\theta}_t, \qquad \dot{\theta}_{t+1} = \dot{\theta} + \frac{g\sin{\theta} + \cos{\theta} (\frac{-u-m_pl\dot{\theta}^2\sin{\theta}}{m_c+m_p})}{l(\frac{4}{3} - \frac{m_p\cos^2{\theta}}{m_c + m_p})}
\end{align}

\vspace{-4mm}
\subsubsection*{Controller details}

Our HDCs \chd were trained in continuous action spaces with the deep deterministic policy gradient (DDPG)~\cite{ddpg}, which is an off-policy, actor-critic algorithm for deep reinforcement learning. 
For the pendulum, the \chd input is a $64 \times 64$ image and an angular velocity $\dot{\theta}$. The structure of \chd consists of convolutional layers (10 for IP, 12 for MC, 12 for CP) containing 400 neurons, fully connected layers, pooling layers, and flattened layers, along with \textit{ReLU} activation functions. for the mountain car, we have the same \chd structure and input dimensions (with velocity $v$ instead of angular velocity $\dot{\theta}$). Although these controllers perform well, they are impractical to verify directly due to their complexity. 

To approximate high-dimensional controllers, we train feedforward neural networks \cld with only low-dimensional state inputs to imitate the performance of the \chd. Step 1 in the training process of \cld is illustrated in Fig.~\ref{fig:ldc-training} and formalized in Alg.~\ref{alg:train}.
The structure of \cld is simpler to enable exhaustive verification: 2 layers with 20 neurons each and \textit{Sigmoid} and \textit{Tanh} activation functions. 

\subsubsection*{Experimental verification details}

\begin{itemize}
    \item In the IP case, the initial set we considered is $S_0^{ip} = \{ (\theta_{0}, \dot{\theta}_{0}) \in [0, 2] \times [-2, 0] \}$.  In $T=30$ steps, we checked whether the rod stays in the target set $G_{ip} = \{ \theta \in [0, 0.35] \}$.

\item In the MC case, the initial set is $S_0^{mc} = $ $\{ (x_{0}, v_{0}) \in [-0.6, -0.4] \times [-0.02, 0.05]\}$. Given $T=60$ steps, we checked whether the mountain car will stay in the target set $G^{ip} = \{ x \in [0.45, \infty] \}$. 

\item In the CP case, the initial set is $S_0^{cp} = \{(x_{0}, v_{0}, \theta_{0}, \dot{\theta}_{0}) \in [0, 0.1] \times [0, 0.1 ] \times [0.05, 0.15] \times [-0.4, -0.35] \}$. Given $T=20$ steps, we checked whether the cartpole would stay in the target set $G^{ip} = \{ \theta \times x \in [-0.2, 0.2] \times [0, 0.2] \}$.  
\end{itemize}

\subsection{Complementary algorithms}\label{sec:sub_algo}

This section provides the algorithms discussed in previous sections.
Algorithm~\ref{alg:train} refers to the first step---knowledge distillation LDC training. Algorithms~\ref{alg:trajectory-conformal} and \ref{alg:action-conformal} represent Step 3 in our end-to-end method---computation of action- and trajectory-based discrepancies. Algorithm \ref{alg:iter-train-2} is the counterpart of Algorithm \ref{alg:iter-train}.

\begin{algorithm}[h]
\caption{Training an LDC with HDC supervision}
\label{alg:train}
\begin{algorithmic}
\Function{TrainLDC}{HDC \chd, image generator $g$, initial state region $\bar{S}$, time horizon $T$, threshold for Lipschitz constant $\lambda$, threshold for MSE $\epsilon$}
\State $X_{init} \gets s_1, s_2, \dots ,s_n \sim \operatorname{Uniform}(\bar{S})$
\For{$i = 1$ to $n$}
    \State $X \gets \tau_{hd}(s_i, T)$
    \Comment{Simulate HDC system to get low-dimen. state}
    \For{$j = 1$ to $T$}
    \State $s_j \gets \varphi_{\mhd}(s_i, j)$
    \State $z_j \gets g(s_j)$
    \State $Y \gets \chd(z_i)$     \Comment{Store HDC control action dataset}
    \EndFor
\EndFor
\State {Training dataset $\mathcal{D}_{tr} \gets (X, Y)$}
\State {$\cld \gets $ two-objective gradient descent($\mathcal{D}_{tr}, \lambda, \epsilon  $) }
\State \textbf{return} \cld
\EndFunction
\end{algorithmic}
\end{algorithm}

\begin{algorithm}[h]

\caption{Computation of trajectory-based discrepancy}
\label{alg:trajectory-conformal}
\begin{algorithmic}
\Function{ComputeTrajDiscr}{LDC $\cld$, HDC $\chd$, image generator $g$, state region $\bar{S}$, confidence $\alpha$, sample count $N$, time steps $T$}
    \State 
    $s_0, s_1, ..., s_N \sim \operatorname{Uniform}(\bar{S})$
        \For{$i = 1$ to $N$}
            \State $\delta_i \gets
            \max_{t=1,...,T}\left\| \varphi_{\mhd}(s_i, t) - \varphi_{\mld}(s_i, t)\right\|_1$ \Comment{\text{non-conformity scores}}

        \EndFor
        \State $r \gets \lceil (N+1)(1-\alpha) \rceil $ \Comment{conformal quantile}
    \State \textbf{return}  $r$-th smallest value among $[\delta_{1}, \ldots, \delta_{N},\infty]$
    \EndFunction
    
\end{algorithmic}
\end{algorithm}

\begin{algorithm}[thb]
\caption{Computation of action-based discrepancy}
\label{alg:action-conformal}
\begin{algorithmic}
\Function{ComputeActionDiscr}{LDC \cld, HDC \chd, image generator $g$, state region $\bar{S}$, confidence $\alpha$, sample count $N$}
    \State 
    $s_0, s_1, ..., s_N \sim \operatorname{Uniform}(\bar{S})$
        \For{$i = 1$ to $N$}
            \State $\delta_i \gets \max_{t=0..T}\left\|\chd(g(\varphi_{\mhd}(s_0, t))) - \cld(\varphi_{\mhd}(s_0, t))\right\|_1$ \Comment{\text{non-conformity scores}}
        \EndFor
        \State $r \gets \lceil (N+1)(1-\alpha) \rceil $ \Comment{conformal quantile}
    \State \textbf{return}  $r$-th smallest value among $[\delta_{1}, \ldots, \delta_{N},\infty]$
    \EndFunction
\end{algorithmic}
\end{algorithm}


\begin{algorithm}[h]
\caption{Iterative LDC training for the trajectory-based approach}
\label{alg:iter-train-2}
\begin{algorithmic}
\Function{IterativeTrainingTB}{HDC $\chd$, image generator $g$, sample count $N$, initial state space $S_0$, confidence $\alpha$, discrepancy threshold $\xi$, time steps $T$}

\State {$\lambda, \epsilon \gets $ initial values}
\State $\mathbf{S}_0 \gets $ initial gridding of $S_0 : S_1, S_2, \dots$

\While{Computing resources last}
\For{$i = 1$ to $|\mathbf{S_0}|$}
\State $\cld^i \gets$ \textsc{TrainLDC}(\chd, g, $S_i$, $\lambda$, $\epsilon$) 
    \State $\delta^i \gets $ \textsc{ComputeTrajDiscr}($\cld^i$, \chd, $g$, $S_i$, $\alpha$, $N$, $T$)    
    \If{$\delta^i > \xi$}
        \State $\epsilon \gets \epsilon/2$
        \Comment{Reduce MSE threshold}
    \EndIf
\EndFor

\If{$\hat{\delta} > \epsilon$ in some sub-regions $\mathbf{\hat{S}} \subseteq \mathbf{S_0}$}
\State  $\mathbf{S_0} \gets \mathbf{S}'_0$ with refined re-gridding of $\hat{S}$ \Comment{Reduce Lipschitz threshold}
\EndIf

\If{$\hat{\delta} \leq \xi \land \rs_{\mld}(\hat{\mathbf{S}}, T)  \not\subseteq G$ in some sub-regions $\mathbf{\hat{S}} \subseteq  \mathbf{S_0}$}
\State  $\lambda \gets \lambda/2$ and keep the same $\epsilon$ in $\hat{\mathbf{S_0}}$
\Comment{Reduce Lipschitz threshold}
\EndIf

\State $\mathbf{S_0} \gets \mathbf{S_0}' $ \Comment{Use the updated grid}
\EndWhile

\State $\bar{\beta} \gets
{\delta^1, \delta^2, \dots}$

\State \textbf{return} $\cld^1, \cld^2, \dots ,\cld^n, \bar{\beta} $
\EndFunction
\end{algorithmic}
\end{algorithm}



\end{document}